\documentclass[letterpaper]{article} % DO NOT CHANGE THIS
\usepackage{aaai2026}  % DO NOT CHANGE THIS
\usepackage{times}  % DO NOT CHANGE THIS
\usepackage{helvet}  % DO NOT CHANGE THIS
\usepackage{courier}  % DO NOT CHANGE THIS
\usepackage[hyphens]{url}  % DO NOT CHANGE THIS
\usepackage{graphicx} % DO NOT CHANGE THIS
\usepackage{natbib}  % DO NOT CHANGE THIS AND DO NOT ADD ANY OPTIONS TO IT
\usepackage{caption} % DO NOT CHANGE THIS AND DO NOT ADD ANY OPTIONS TO IT
\usepackage{algorithm}
\usepackage{algorithmic}

\usepackage{newfloat}
\usepackage{listings}
\DeclareCaptionStyle{ruled}{labelfont=normalfont,labelsep=colon,strut=off} % DO NOT CHANGE THIS
\floatstyle{ruled}
\newfloat{listing}{tb}{lst}{}
\floatname{listing}{Listing}
\usepackage{amsmath}
\usepackage{amssymb}
\usepackage{amsthm}

\newtheorem{theorem}{Theorem}
\newtheorem{lemma}{Lemma}
\newtheorem{corollary}{Corollary}
\newtheorem{assumption}{Assumption}

\usepackage{multirow}
\usepackage{booktabs}

\usepackage{makecell}

\title{T-SKM-Net: Trainable Neural Network Framework for Linear Constraint Satisfaction via Sampling Kaczmarz-Motzkin Method}
\author {
    Haoyu Zhu\textsuperscript{\rm 1},
    Yao Zhang\textsuperscript{\rm 1},
    Jiashen Ren\textsuperscript{\rm 1},
    Qingchun Hou\textsuperscript{\rm 1}\thanks{Corresponding author.}
}
\affiliations {
    \textsuperscript{\rm 1}Zhejiang University\\
    haoyu.22@intl.zju.edu.cn, eezhangyao@zju.edu.cn, jiashen.22@intl.zju.edu.cn, houqingchun@zju.edu.cn
}

\begin{document}

\maketitle

\begin{abstract}
Neural network constraint satisfaction is crucial for safety-critical applications such as power system optimization, robotic path planning, and autonomous driving. However, existing constraint satisfaction methods face efficiency-applicability trade-offs, with hard constraint methods suffering from either high computational complexity or restrictive assumptions on constraint structures. The Sampling Kaczmarz-Motzkin (SKM) method is a randomized iterative algorithm for solving large-scale linear inequality systems with favorable convergence properties, but its argmax operations introduce non-differentiability, posing challenges for neural network applications. This work proposes the Trainable Sampling Kaczmarz-Motzkin Network (T-SKM-Net) framework and, for the first time, systematically integrates SKM-type methods into neural network constraint satisfaction. The framework transforms mixed constraint problems into pure inequality problems through null space transformation, employs SKM for iterative solving, and maps solutions back to the original constraint space, efficiently handling both equality and inequality constraints. We provide theoretical proof of post-processing effectiveness in expectation and end-to-end trainability guarantees based on unbiased gradient estimators, demonstrating that despite non-differentiable operations, the framework supports standard backpropagation. On the DCOPF case118 benchmark, our method achieves 4.27ms/item GPU serial forward inference with 0.0025\% max optimality gap with post-processing mode and 5.25ms/item with 0.0008\% max optimality gap with joint training mode, delivering over 25$\times$ speedup compared to the pandapower solver while maintaining zero constraint violations under given tolerance.
\end{abstract}

\begin{links}
    \link{Code}{https://github.com/IDO-Lab/T-SKM-Net}
    % \link{Datasets}{https://aaai.org/example/datasets}
    % \link{Extended version}{https://github.com/IDO-Lab/T-SKM-Net/blob/main/paper.pdf}
\end{links}

\section{Introduction}

Constrained decision problems arise across various disciplines, such as the DC Optimal Power Flow (DCOPF) problem in power systems~\cite{carpentier1962contribution}, which requires minimizing operational costs while satisfying physical and security constraints of the power system. However, directly solving these problems using optimization solvers requires substantial computational time. Therefore, in scenarios demanding rapid or even real-time responses, traditional solvers often fail to meet timing requirements~\cite{scutari2018parallel}, motivating researchers to explore more efficient approximate solution methods.

In recent years, deep learning~\cite{goodfellow2016deep} has demonstrated powerful function approximation capabilities across various complex tasks~\cite{hornikMultilayerFeedforwardNetworks1989,cybenkoApproximationSuperpositionsSigmoidal1989,lecunDeepLearning2015}, providing new insights for solving constrained optimization problems~\cite{smithNeuralNetworksCombinatorial1999,JMLR:v24:21-0449,houGeneralizeLearnedHeuristics2023,liu2024teaching}. Constrained decision problems can be transformed into single forward inference of neural networks by learning the mapping relationship from problem parameters to decision variables. However, neural network outputs often cannot guarantee strict satisfaction of the original constrained decision problem's constraints, which limits their use in safety-critical applications.

To address this challenge, researchers have proposed various methods for integrating constraints into neural networks. Existing methods can be mainly categorized into soft constraints and hard constraints. Soft constraint methods~\cite{raissiPhysicsinformedNeuralNetworks2019} indirectly handle constraints by incorporating constraint violation terms as penalty terms in the loss function. While these methods are simple to implement and maintain network differentiability, they cannot strictly guarantee constraint satisfaction, posing potential risks in safety-critical applications. Hard constraint methods attempt to strictly satisfy constraints at network output, including differentiable optimization layers~\cite{amosOptNetDifferentiableOptimization2017,dontiDC3LearningMethod2021,minHardNetHardConstrainedNeural2025}, parameterized feasible space methods~\cite{tordesillasRAYENImpositionHard2023,zhangEfficientLearningBasedSolver2024}, and decision rule approaches~\cite{constante-floresEnforcingHardLinear2025}. However, these methods still face challenges such as high computational complexity, requirement for pre-computing feasible points, or limited expressiveness when handling input-dependent dynamic linear constraints.

To address these limitations, this work proposes the Trainable Sampling Kaczmarz-Motzkin Network (T-SKM-Net) framework. The main contributions include: (1) First integration of the Sampling Kaczmarz-Motzkin method to neural network linear constraint satisfaction for input-dependent dynamic constraints; (2) Theoretical proof that this method can serve as an effective approximation of $L_2$ projection, providing theoretical foundation for its use as a post-processing strategy; (3) Theoretical demonstration of end-to-end trainability of the framework through unbiased gradient estimators~\cite{robbins1951stochastic,bottou2010large,goda2023constructing}, addressing training challenges caused by non-differentiable argmax operations.

\section{Related Work} 

Neural network constraint satisfaction methods can be categorized into soft constraints and hard constraints based on constraint handling approaches, where hard constraint methods can be further subdivided into three main paradigms. 

\subsection{Soft Constraint Methods} 

Soft constraint strategies indirectly handle constraints by incorporating constraint violation terms as penalty terms in the loss function~\cite{raissiPhysicsinformedNeuralNetworks2019,doi:10.1137/18M1225409}. While these methods are simple to implement and preserve network differentiability, they cannot strictly guarantee constraint satisfaction, posing potential risks in safety-critical applications. 

\subsection{Hard Constraint Methods} 

\subsubsection{Differentiable Optimization Layers and Projection Methods} 

These methods ensure constraint satisfaction by embedding optimization problems into neural networks or employing projection operations~\cite{chenEnforcingPolicyFeasibility2021}. OptNet~\cite{amosOptNetDifferentiableOptimization2017} embeds quadratic programming layers~\cite{butlerEfficientDifferentiableQuadratic2023} into neural networks but suffers from high computational complexity. DC3~\cite{dontiDC3LearningMethod2021} adopts completion and correction strategies but cannot strictly guarantee equality constraint satisfaction due to linear approximations. HardNet~\cite{minHardNetHardConstrainedNeural2025} provides closed-form projection expressions but relies on restrictive assumptions such as constraint matrix full rank. GLinSAT~\cite{zengGLinSATGeneralLinear2024} solves an entropy-regularized linear program with an accelerated first-order method to impose general linear constraints, but still requires an inner iterative solver per batch. Recent works also design feasibility-seeking or projection-like layers with guarantees, such as homeomorphic projection for sets homeomorphic to a unit ball (including all compact convex sets and some nonconvex sets)~\cite{liang2023low}, feasibility-seeking NNs via unrolled violation minimization~\cite{nguyen2025fsnet}, and feasibility-restoration mappings for AC-OPF~\cite{hanFRMNetFeasibilityRestoration2024a}, but they introduce additional optimization or invertible-network components inside the layer.

\subsubsection{Parameterized Feasible Space Methods} 

These methods parameterize neural network outputs to feasible regions. \citeauthor{zhangEfficientLearningBasedSolver2024}~\shortcite{zhangEfficientLearningBasedSolver2024} uses gauge maps for polyhedra mapping but requires the origin to be a strict interior point, making equality constraint handling difficult. RAYEN~\cite{tordesillasRAYENImpositionHard2023} employs geometric transformations but requires offline computation of feasible points, limiting applicability to input-dependent constraints. 

\subsubsection{Decision Rule Methods} 

Recent work introduces decision rule-based methods from stochastic optimization. \cite{constante-floresEnforcingHardLinear2025} et al.\ propose a framework combining task and safety networks through convex combinations. However, this approach has limited expressiveness and requires convex constraint sets. Preventive learning~\cite{zhao2023ensuring} calibrates linear inequality constraints during training to anticipate DNN prediction errors and ensure feasibility without post-processing, but is tailored to convex linear constraints. 

\subsection{Sampling Kaczmarz-Motzkin Method} 

The Kaczmarz method~\cite{karczmarz1937angenaherte} and Motzkin relaxation method~\cite{motzkin1954relaxation} are classical iterative techniques for solving linear systems and inequalities, respectively. \citeauthor{strohmerRandomizedKaczmarzAlgorithm2009}~\shortcite{strohmerRandomizedKaczmarzAlgorithm2009} proved exponential convergence for the randomized Kaczmarz algorithm. \citeauthor{loeraSamplingKaczmarzMotzkinAlgorithm2017}~\shortcite{loeraSamplingKaczmarzMotzkinAlgorithm2017} unified these approaches in the Sampling Kaczmarz-Motzkin (SKM) method for large-scale linear inequality systems. \citeauthor{morshedSamplingKaczmarzMotzkin2022}~\shortcite{morshedSamplingKaczmarzMotzkin2022} further improved SKM with global linear convergence guarantees. Despite favorable theoretical properties, SKM-type methods have not been integrated to neural network constraint satisfaction problems. 

\subsection{Positioning of Our Contributions} 

Existing methods have various limitations when handling input-dependent dynamic linear constraints: soft constraints cannot strictly guarantee constraint satisfaction; differentiable optimization layers and feasibility-seeking architectures (e.g., GLinSAT~\cite{zengGLinSATGeneralLinear2024}, homeomorphic projection~\cite{liang2023low}, FSNet~\cite{nguyen2025fsnet}, FRMNet~\cite{hanFRMNetFeasibilityRestoration2024a}, DC3~\cite{dontiDC3LearningMethod2021}) often require solving auxiliary optimization problems or running invertible-network subroutines in each forward pass; parameterized feasible space methods require prior knowledge or offline computation of feasible points, making them difficult to handle dynamic constraints; decision rule methods have limited expressiveness and typically require convex constraint sets.

The T-SKM-Net framework proposed in this work first introduces SKM methods to the neural network constraint satisfaction domain, leveraging their computational efficiency and theoretical convergence guarantees. Unlike existing methods, the T-SKM-Net framework provides two flexible usage modes:

\begin{enumerate}
\item \textbf{Post-processing Method}: Used solely as a post-processing method without incorporation into training steps. It can be directly applied to any pre-trained neural network to ensure satisfaction of linear constraints.
\item \textbf{Joint Training Mode}: Integrates the constraint satisfaction layer into the neural network framework, achieving end-to-end joint optimization that optimizes prediction performance while satisfying constraints.
\end{enumerate}

\section{Preliminaries}

\subsection{Problem Formulation}

Consider the mixed linear constraint system:
\begin{align}
A(x)z &\leq b(x) \\
C(x)z &= d(x)
\end{align}
where $x \in \mathbb{R}^{n_{\text{in}}}$ is the input, $z \in \mathbb{R}^{n}$ is the output variable, $A(x) \in \mathbb{R}^{p\times n}$, $C(x) \in \mathbb{R}^{q\times n}$ are constraint matrices, and $b(x) \in \mathbb{R}^p$, $d(x) \in \mathbb{R}^q$ are right-hand side vectors. The feasible region is:
\begin{align}
\mathcal{F}(x) = \{z \in \mathbb{R}^n: A(x)z\leq b(x), C(x)z=d(x) \}
\end{align}

In neural network constraint satisfaction, an upstream network $f_{\theta}: \mathbb{R}^{n_{\text{in}}} \to \mathbb{R}^n$ produces output $y_{0}=f_{\theta}(x)$ that typically violates constraints ($y_{0} \notin \mathcal{F}(x)$). A constraint satisfaction layer transforms $y_{0}$ into a feasible solution $z^* \in \mathcal{F}(x)$. The layer must satisfy: (1) constraint satisfaction, (2) solution quality, (3) computational efficiency, and (4) end-to-end trainability.

\subsection{Sampling Kaczmarz-Motzkin Method}

The SKM algorithm solves linear inequality systems $Az\leq b$ through the following iteration:
\begin{align}
z_{k+1} = z_{k} - \delta \frac{\left(a_{i^*}^\top z_{k}-b_{i^*}\right)_{+}}{\|a_{i^*}\|^2} a_{i^*}
\end{align}

where $\delta>0$ is the step size, $i^*=\arg\max_{i\in S_k}\left(a_{i}^\top z_{k}-b_{i}\right)_{+}$ is the most violated constraint in the sampled set $S_{k} \subseteq \{1,\ldots,p\}$ with $|S_{k}| = \beta$, $(\cdot)_{+}=\max(\cdot,0)$, and $a_{i}$ is the $i$-th row of matrix $A$.

\section{T-SKM-Net Framework}

\subsection{Framework Overview}

T-SKM-Net (Trainable Sampling Kaczmarz-Motzkin Network) addresses the fundamental challenge of efficiently handling mixed linear constraint systems in neural networks. While SKM methods excel at solving pure inequality systems, directly applying them to mixed constraints faces significant geometric challenges: equality constraints define hyperplanes (zero-measure sets) while inequality constraints define half-spaces, creating a fundamental mismatch that leads to inefficient oscillations between constraint types during iterative processing.

\begin{figure}[ht]
\centering
\includegraphics[width=\columnwidth]{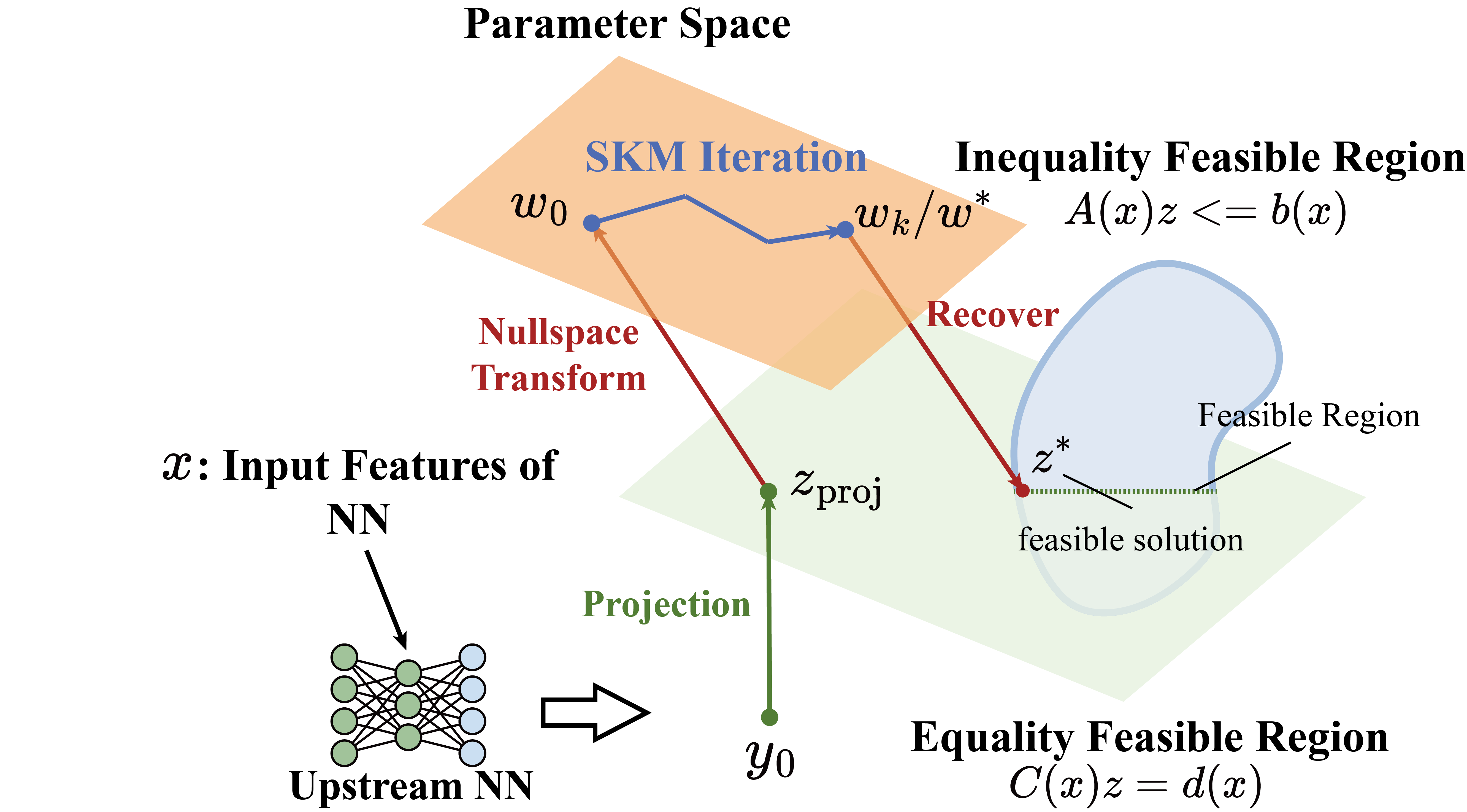}
\caption{T-SKM-Net framework architecture.}
\label{fig:framework}
\end{figure}

To address this challenge, T-SKM-Net transforms mixed constraint problems into pure inequality problems before applying SKM iterations. The framework accepts upstream neural network output $y_{0} \in \mathbb{R}^n$ and input $x \in \mathbb{R}^{n_{\text{in}}}$, producing output $z^* \in \mathcal{F}(x)$ that strictly satisfies all constraints through three key steps (Figure~\ref{fig:framework}):

(1) \textbf{Constraint Transformation}: Convert mixed constraints into pure inequalities through SVD-based null space transformation~\cite{golub2013matrix}, which addresses the geometric mismatch by decomposing the problem into equality-satisfying subspace and inequality optimization; (2) \textbf{SKM Iteration}: Apply randomized constraint selection and projection in the transformed space, where iterations effectively move within the equality constraint manifold while seeking inequality constraint satisfaction; (3) \textbf{Solution Recovery}: Map the solution back to the original constraint space while preserving both equality and inequality constraint satisfaction.

% Mathematically, the framework is represented as $z^* = \mathcal{T}\text{-SKM}(y_{0},x,\omega)$, where $\omega$ represents the random sampling sequence.

\subsection{Algorithm Pipeline}

\begin{algorithm}[h]
\caption{T-SKM-Net Constraint Satisfaction}
\label{alg:tskm}
\begin{algorithmic}[1]
\REQUIRE Upstream output $y_0$, constraints\\ $A(x), b(x), C(x), d(x)$
\ENSURE Feasible solution $z^*$

\STATE \textbf{Constraint Transformation:}
\STATE Compute SVD: $C = U\Sigma V^T$, obtain $N$ (null space basis)
\STATE $z_{\text{proj}} = y_0 - C^{\dagger}(Cy_0 - d)$
\STATE $A_{\text{new}} = AN$, $b_{\text{new}} = b - Az_{\text{proj}}$

\STATE \textbf{SKM Iteration:}
\STATE Initialize $w_0 = 0$
\FOR{$k = 0, 1, \ldots, K-1$}
    \STATE Sample constraint set $S_k \subseteq \{1,\ldots,p\}$ with $|S_k| = \beta$
    \STATE $i^* = \arg\max_{i \in S_k} (a_i^T w_k - b_i^{\text{new}})_+$
    \STATE $w_{k+1} = w_k - \delta \frac{(a_{i^*}^T w_k - b_{i^*}^{\text{new}})_+}{\|a_{i^*}\|^2} a_{i^*}$
    \IF{termination condition met}
        \STATE \textbf{break}
    \ENDIF
\ENDFOR

\STATE \textbf{Solution Recovery:}
\STATE $z^* = z_{\text{proj}} + Nw_K$
\RETURN $z^*$
\end{algorithmic}
\end{algorithm}

The algorithm ensures constraint satisfaction: equality constraints are satisfied by construction since $Nw_K \in \text{null}(C)$, while inequality constraints are satisfied when SKM converges.

\subsection{Theoretical Guarantees}

We provide three theoretical results establishing the effectiveness and trainability of T-SKM-Net:

\begin{theorem}[SKM L2 Projection Approximation]
\label{thm:skm-l2}
For inequality constraint $Az\leq b$ with feasible region $\mathcal{P}$, the SKM method starting from $z_0$ satisfies:
\begin{align}
\mathbb{E}[d(z_{k},z_{0})] \leq 2 \cdot d(z_{0},\mathcal{P})
\end{align}

where $d(z_0, \mathcal{P})$ is the distance from initial point to feasible region.
\end{theorem}

\textbf{Proof Sketch:} We construct an auxiliary function $V(z) = \|z - P(z_0)\|^2$ where $P(z_0)$ is the L2 projection onto $\mathcal{P}$. The key insight is that SKM updates exhibit Fejér monotonicity~\cite{combettes2008fejer}: $V(z_{k+1}) \leq V(z_k)$ for $0 < \delta < 2$. Applying Jensen's inequality and the triangle inequality yields the bound. \textit{Complete proof in appendix.}

\begin{theorem}[T-SKM-Net L2 Projection Approximation]
\label{thm:tskm-l2}
For mixed constraint system with feasible region $\mathcal{F}(x)$ and null space basis matrix $N$, T-SKM-Net satisfies:
\begin{align}
\mathbb{E}[d(z_k, y_0)] \leq \sqrt{1+4\kappa(N)^2} \cdot d(y_0, P(y_0))
\end{align}
where $\kappa(N)$ is the condition number of $N$ and $P(y_0)$ is the L2 projection. For SVD-based $N$, $\kappa(N)=1$, yielding $\mathbb{E}[d(z_k, y_0)] \leq \sqrt{5} \cdot d(y_0, P(y_0))$.
\end{theorem}

\textbf{Proof Sketch:} The framework decomposes the original problem through null space transformation: $z = z_{\text{proj}} + Nw$. We exploit orthogonality between the projection error $e = y_0 - z_{\text{proj}}$ and the null space component $Nw$. Combined with Theorem~\ref{thm:skm-l2} applied to the transformed subproblem and matrix singular value bounds, the Cauchy-Schwarz inequality yields the final bound where the condition number $\kappa(N)$ controls the approximation quality. \textit{Complete proof in appendix.}

\begin{assumption}[Regularity Conditions]
\label{assump:regularity}
The parameterized SKM algorithm with constraints $A(x)w \leq b(x)$ satisfies:
\begin{enumerate}
\item \textbf{Non-degeneracy}: $\|a_i(x)\| \geq c > 0$ for all $i,x$
\item \textbf{Differentiability}: $A(x), b(x)$ are differentiable\\ with $\|\nabla_x A(x)\| \leq L_A$, $\|\nabla_x b(x)\| \leq L_b$
\item \textbf{Boundedness}: $\|A(x)\|, \|b(x)\| \leq M$ for all relevant $x$
\item \textbf{Step size}: $0 < \delta < 2$
\item \textbf{Non-degeneracy of tie events}: Constraint violation differences do not vanish identically on open subsets of input space
\end{enumerate}
\end{assumption}

Let $W_K$ be a random function mapping $x$ to the result after $K$ steps of random SKM iterate with sampling path $\omega$.

\begin{theorem}[End-to-End Trainability]
\label{thm:trainability}
Under Assumption~\ref{assump:regularity}, the SKM algorithm supports end-to-end training with: (i) well-defined expected gradients $\nabla_x \mathbb{E}_\omega[W_K(x,\omega)]$, (ii) finite gradient variance $\text{Var}(\nabla_x W_K(x,\omega)) < \infty$, and (iii) unbiased gradient estimation $\mathbb{E}_\omega[\nabla_x W_K(x,\omega)] = \nabla_x \mathbb{E}_\omega[W_K(x,\omega)]$.
\end{theorem}

\textbf{Proof Sketch:} The main challenge is handling non-differentiable $\arg\max$ operations in SKM. We establish that tie events (where multiple constraints achieve the maximum violation) have zero Lebesgue measure under Assumption~\ref{assump:regularity}(5). On the complement, SKM is piecewise differentiable with bounded Jacobians. The gradient bound $\|\nabla_x W_k\| \leq kC$ ensures dominated convergence, enabling interchange of expectation and differentiation. Since sampling probabilities are independent of $x$, gradients are unbiased. \textit{Complete proof in appendix.}

These results establish that T-SKM-Net produces high-quality solutions in expectation (Theorems~\ref{thm:skm-l2} and \ref{thm:tskm-l2}) while supporting standard backpropagation despite non-differentiable operations (Theorem~\ref{thm:trainability}).

\subsection{Usage Modes}

T-SKM-Net supports two flexible usage modes:

\textbf{Post-processing Mode}: Applied to pre-trained networks without modifying parameters. Given network $f_\theta$ and input $x$, compute $z^* = \mathcal{T}\text{-SKM}(f_\theta(x), x, \omega)$, ensuring constraint satisfaction while preserving network accuracy.

\textbf{End-to-End Training Mode}: Integrated into network architecture for joint optimization. The training objective becomes:
\begin{align}
\min_\theta \mathbb{E}[\mathcal{L}(\mathcal{T}\text{-SKM}(f_\theta(x), x, \omega), y_{\text{true}})]
\end{align}
Despite non-differentiable $\arg\max$ operations, Theorem~\ref{thm:trainability} guarantees unbiased gradient estimation by treating the sampling sequence $\omega$ as a deterministic path during backpropagation.

\subsection{Implementation and Strategy}

\subsubsection{Computational Optimization}

\textbf{Batch tensorization}: T-SKM-Net is implemented as a batched tensor operator, so all steps (null space transform, SKM iteration, recovery) can be easily parallelized on GPUs for large-scale problems.

\textbf{SVD precomputation}: In applications where the equality matrix $C$ is fixed (e.g., DCOPF with fixed topology), we precompute $C = U\Sigma V^\top$ once offline and reuse the null-space basis $N$, which removes this cost from the online path.

\subsubsection{Parameter Selection Guidelines}

\textbf{Step size and sampling}: In all experiments we set $\delta=1.0$ and observe a good trade-off between convergence speed and projection accuracy. For the sampling size $\beta$, we use small fixed values (e.g., 5–10) on small problems and scale it roughly like $\beta = \mathcal{O}(\sqrt{m})$ on large problems, where $m$ is the number of inequality constraints.

\subsubsection{Training Strategy Optimization}

\textbf{Violation-aware loss}: For joint training, we add simple penalty terms on the equality/inequality violations of $f_\theta(x)$ to the task loss, which empirically shortens SKM convergence and stabilizes training.

\textbf{Delayed activation}: We often freeze the T-SKM-Net layer in the early epochs and enable it only in the last 10\%–20\% of training, letting the backbone first learn a reasonable initialization before enforcing hard constraints.

\subsubsection{Algorithm Variant Support}

T-SKM-Net is compatible with SKM variants (e.g., momentum-accelerated updates).

\section{Experiments}

We evaluate T-SKM-Net on synthetic constraint satisfaction problems and real-world DC Optimal Power Flow (DCOPF) applications. All experiments are conducted on Intel Core i5-13500 CPU and AMD Radeon Instinct MI50 GPU (if not specified). We compare against state-of-the-art optimization solvers and neural constraint satisfaction methods, focusing on computational efficiency, solution quality, and constraint satisfaction.

\subsection{L2 Projection Efficiency}

We first evaluate T-SKM-Net as a post-processing method for L2 projection onto feasible regions. We construct random linear constraint systems with equal numbers of equality and inequality constraints (both $n/2$ where $n$ is the variable dimension), starting from infeasible points with initial maximum violations on the order of $10^2$.

We compare against three optimization approaches: CVXPY (using OSQP algorithm~\cite{stellato2020osqp}), OSQP library directly, and Gurobi commercial solver. Tables~\ref{tab:projection} and \ref{tab:projection_large} show results where "Speedup" indicates acceleration relative to the baseline method, "Max Vio." denotes maximum constraint violation, and "Rel. Err." measures the L2 distance to CVXPY's solution.

\begin{table}[t]
\centering
\setlength{\tabcolsep}{1mm}
\begin{tabular}{llcccc}
\toprule
\textbf{Dim} & \textbf{Method} & \textbf{Time (s)} & \textbf{Speedup} & \textbf{\makecell{Max\\Vio.}} & \textbf{\makecell{Rel.\\Err.}} \\
\midrule
\multirow{4}{*}{500} & CVXPY & 0.164 & 1.0$\times$ & 1.78e-14 & - \\
 & OSQP & 0.435 & 0.38$\times$ & 9.18e-06 & 0.74 \\
 & Gurobi & 0.073 & 2.24$\times$ & 1.69e-14 & 0.0 \\
 & T-SKM-Net & \textbf{0.017} & \textbf{9.65$\times$} & 4.30e-06 & 12.1 \\
\midrule
\multirow{4}{*}{2000} & CVXPY & 10.03 & 1.0$\times$ & 4.62e-14 & - \\
 & OSQP & 29.80 & 0.34$\times$ & 5.90e-05 & 0.36 \\
 & Gurobi & 1.658 & 6.05$\times$ & 3.20e-14 & 0.0 \\
 & T-SKM-Net & \textbf{0.237} & \textbf{42.3$\times$} & 1.19e-06 & 7.53 \\
\midrule
\multirow{4}{*}{3000} & CVXPY & 32.59 & 1.0$\times$ & 5.51e-14 & - \\
 & OSQP & 96.24 & 0.34$\times$ & 1.13e-04 & 0.44 \\
 & Gurobi & 4.217 & 7.73$\times$ & 6.04e-14 & 0.0 \\
 & T-SKM-Net & \textbf{0.856} & \textbf{38.1$\times$} & 6.20e-07 & 6.19 \\
\bottomrule
\end{tabular}
\caption{L2 Projection Performance (Medium Scale, Speedup vs CVXPY)}
\label{tab:projection}
\end{table}

\begin{table}[t]
\centering
\setlength{\tabcolsep}{1.5mm}
\begin{tabular}{llccc}
\toprule
\textbf{Dim} & \textbf{Method} & \textbf{Time (s)} & \textbf{Speedup} & \textbf{Max Vio.} \\
\midrule
\multirow{2}{*}{4000} & Gurobi & 9.184 & 1.0$\times$ & 7.82e-14 \\
 & T-SKM-Net & \textbf{2.001} & \textbf{4.59$\times$} & 8.71e-06 \\
\midrule
\multirow{2}{*}{8000} & Gurobi & 60.76 & 1.0$\times$ & 1.07e-13 \\
 & T-SKM-Net & \textbf{13.30} & \textbf{4.57$\times$} & 3.41e-07 \\
\midrule
\multirow{2}{*}{10000} & Gurobi & 103.6 & 1.0$\times$ & 1.72e-13 \\
 & T-SKM-Net & \textbf{22.82} & \textbf{4.54$\times$} & 1.31e-06 \\
\bottomrule
\end{tabular}
\caption{L2 Projection Performance (Large Scale, Speedup vs Gurobi)}
\label{tab:projection_large}
\end{table}

T-SKM-Net achieves significant speedups while maintaining approximation errors orders of magnitude smaller than initial violations. For large-scale problems where open-source solvers become impractical, T-SKM-Net consistently outperforms the commercial Gurobi solver with 4-5$\times$ speedups.

\subsection{Constrained Optimization Problems}

We evaluate T-SKM-Net on constrained quadratic programming problems to assess both post-processing effectiveness and end-to-end trainability. We consider the optimization problem~\cite{nocedal2006numerical}:
\begin{align}
\arg\min_y \quad & \frac{1}{2} y^T Q y + p^T y \\
\text{s.t.} \quad & Ay = x \\
& Gy \leq h
\end{align}
where $y \in \mathbb{R}^{100}$ is the decision variable, $x \in \mathbb{R}^{50}$ is the input parameter, and we vary the number of inequality constraints from 10 to 150.

We compare T-SKM-Net against optimization solvers (OSQP, qpth~\cite{amosOptNetDifferentiableOptimization2017}), differentiable optimization methods (OptNet), and neural constraint satisfaction methods (DC3, gradient-based post-processing). Neural networks are trained using gradient descent with constraint violation penalties in the loss function to approximate the optimization problem solutions. The neural network baseline uses gradient-based correction at test time to reduce constraint violations~\cite{dontiDC3LearningMethod2021}. We evaluate both end-to-end training and post-processing modes of T-SKM-Net.

\begin{table}[t]
\centering
\setlength{\tabcolsep}{1mm}
\begin{tabular}{llcccc}
\toprule
\textbf{Ineq.} & \textbf{Method} & \textbf{\makecell{Obj.\\Val.}} & \textbf{\makecell{Max\\Eq.$^{\dagger}$}} & \textbf{\makecell{Max\\Ineq.$^{\dagger}$}} & \textbf{Time (ms)} \\
\midrule
\multirow{7}{*}{30} & OSQP & -16.33 & 0.000 & 0.000 & \textbf{0.3} \\
 & qpth & -16.33 & 0.000 & 0.000 & 86.4 \\
 & OptNet & \textbf{-14.03} & 0.000 & 0.000 & 71.5 \\
 & DC3 & -\textbf{14.02} & 0.000 & 0.000 & 3.3 \\
 & NN, $\leq$ test & \textbf{-14.09} & 0.350$^{\dagger}$ & 0.000 & 2.5 \\
 & Ours (Train) & \textbf{-14.04} & 0.000 & 0.000 & 1.0 \\
 & Ours (Post) & \textbf{-14.05} & 0.000 & 0.000 & 1.0 \\
\midrule
\multirow{7}{*}{50} & OSQP & -15.05 & 0.000 & 0.001$^{\dagger}$ & \textbf{0.8} \\
 & qpth & -15.05 & 0.000 & 0.000 & 106.3 \\
 & OptNet & -12.46 & 0.000 & 0.000 & 101.0 \\
 & DC3 & \textbf{-13.44} & 0.000 & 0.000 & 5.7 \\
 & NN, $\leq$ test & -12.55 & 0.351$^{\dagger}$ & 0.000 & 2.5 \\
 & Ours (Train) & -12.53 & 0.000 & 0.000 & \textbf{0.9} \\
 & Ours (Post) & -12.52 & 0.000 & 0.000 & \textbf{0.9} \\
\midrule
\multirow{7}{*}{130} & OSQP & -12.73 & 0.000 & 0.000 & \textbf{1.1} \\
 & qpth & -12.73 & 0.000 & 0.000 & 492.8 \\
 & OptNet & \textbf{-10.25} & 0.000 & 0.000 & 450.6 \\
 & DC3 & -9.45 & 0.000 & 0.002$^{\dagger}$ & 109.0 \\
 & NN, $\leq$ test & \textbf{-10.32} & 0.352$^{\dagger}$ & 0.000 & 2.6 \\
 & Ours (Train) & \textbf{-10.23} & 0.000 & 0.000 & \textbf{1.0} \\
 & Ours (Post) & \textbf{-10.24} & 0.000 & 0.000 & 2.1 \\
\midrule
\multirow{7}{*}{150} & OSQP & -12.67 & 0.000 & 0.001$^{\dagger}$ & \textbf{1.2} \\
 & qpth & -12.67 & 0.000 & 0.000 & 652.1 \\
 & OptNet & \textbf{-10.22} & 0.000 & 0.000 & 560.0 \\
 & DC3 & -9.06 & 0.000 & 0.006$^{\dagger}$ & 111.2 \\
 & NN, $\leq$ test & \textbf{-10.29} & 0.351$^{\dagger}$ & 0.000 & 2.5 \\
 & Ours (Train) & \textbf{-10.24} & 0.000 & 0.000 & 1.8 \\
 & Ours (Post) & \textbf{-10.24} & 0.000 & 0.000 & \textbf{1.4} \\
\bottomrule
\end{tabular}
\caption{Constrained QP Performance (100 variables, 50 equality constraints). $^{\dagger}$ indicates constraint violations.}
\label{tab:qp}
\end{table}

Table~\ref{tab:qp} shows results as the number of inequality constraints ("Ineq.") varies from 30 to 150, where "Obj. Val." denotes the objective value, "Max Eq./Ineq." the maximum equality/inequality violations, and $^{\dagger}$ indicates constraint violations. 

The results reveal distinct performance characteristics across different approaches. Traditional optimization solvers (OSQP, qpth) provide exact solutions but with very different costs: qpth’s differentiable implementation incurs substantial overhead (86–652ms). OptNet, which uses qpth as its underlying differentiable optimizer, maintains strict constraint satisfaction, but its high cost makes it impractical for real-time applications. 

Among neural constraint satisfaction methods, T-SKM-Net consistently achieves zero constraint violations in both training and post-processing modes, while DC3 exhibits inequality violations in high-constraint scenarios. The neural network baseline with gradient-based correction, despite achieving competitive objective values, shows persistent equality violations, rendering it unsuitable for safety-critical applications. 

In terms of solution quality, T-SKM-Net is competitive across all constraint densities. For moderate problems (30–50 inequalities), our method attains objective values comparable to NN and OptNet while maintaining zero violations. In high-constraint scenarios (130–150 inequalities), T-SKM-Net matches OptNet’s solution quality while achieving over 200$\times$ speedup, and clearly outperforms DC3, which suffers from both constraint violations and deteriorated objectives. 

The computational efficiency analysis shows that T-SKM-Net achieves large speedups compared to differentiable optimization methods: up to 450$\times$ faster than OptNet and 109$\times$ faster than DC3’s reported times, while keeping runtime comparable to the highly optimized OSQP solver. Most importantly, unlike baseline neural networks that sacrifice constraint satisfaction for speed, T-SKM-Net enforces strict feasibility without compromising computational efficiency. These results demonstrate T-SKM-Net’s effectiveness in both usage modes: as a post-processing method, it transforms infeasible neural network outputs into high-quality feasible solutions; in end-to-end training, it enables joint optimization of prediction accuracy and constraint satisfaction, making it particularly suitable for real-time constrained optimization.

\subsection{DC Optimal Power Flow}

We evaluate T-SKM-Net on the IEEE 118-bus DC Optimal Power Flow (DCOPF) problem~\cite{houLinearizedModelActive2018,anderson2022real}, a critical real-world application in power system operations. This experiment is run on NVIDIA RTX 4090. DCOPF involves minimizing generation costs while satisfying power balance equations and transmission line capacity constraints.

This problem represents a mixed linear constraint system that is well-suited for evaluating our constraint satisfaction approach. The DCOPF problem can be formulated as:

\begin{align}
\min_{P_G,\;\theta}\quad 
& \sum_{g\in G}\Bigl(\tfrac12\,c_g\,P_{G,g}^2 + b_g\,P_{G,g}\Bigr) \\
\text{s.t.}\quad
& P_{G,i} - P_{D,i}
  = \sum_{j\in \mathcal{N}(i)} B_{ij}\,(\theta_i - \theta_j),
  \quad \forall\,i\in \mathcal{B}, \notag\\
& P_{G,i}^{\min} \le P_{G,i} \le P_{G,i}^{\max},
  \quad \forall\,i\in G, \notag\\
-&P_{ij}^{\max} \le B_{ij}\,(\theta_i - \theta_j) \le P_{ij}^{\max},
  \quad \forall\,(i,j)\in \mathcal{L}. \notag
\end{align}

where $G$, $\mathcal{B}$, and $\mathcal{L}$ denote the sets of generators, buses, and transmission lines, respectively, $B_{ij}$ are line susceptances, and the input load demands $P_D$ determine the optimal generation dispatch $P_G$ and voltage angles $\theta$.

We evaluate T-SKM-Net in post-processing mode, joint-training mode, and joint-training mode with MSE loss punishment to the upstream output to verify constraint satisfaction for neural network predictions and compare speed and solution quality under different setups.

We compare against the pandapower~\cite{thurnerPandapowerOpenSourcePython2018} optimization solver as the ground truth baseline, a neural network baseline without constraint satisfaction, and DC3~\cite{dontiDC3LearningMethod2021}. All experiments use single-instance processing (batch size = 1) to reflect real-time operational requirements. Our method under all 3 setups achieves zero constraint violations under $10^{-3}$ tolerance while significantly outperforming existing approaches in both speed and solution quality.
\begin{table}[ht]
\centering
\setlength{\tabcolsep}{1mm}
\begin{tabular}{lcccc}
\toprule
\textbf{Method} & \textbf{Time (ms)} & \textbf{\makecell{Opt. Gap\\(\%)}} & \textbf{\makecell{Max\\Eq.\\(e-4)}} & \textbf{\makecell{Max\\Ineq.\\(e-4)}} \\
\midrule
pandapower & 141.08 & 0.00 (0.00) & 0 & 0 \\
NN & \textbf{0.29} & 4e-4 (3.5e-2) & 2997$^{\dagger}$ & 6$^{\dagger}$ \\
DC3 & 52.15 & 5.6e-5 (2.6e-3) & 0 & 4$^{\dagger}$ \\
Ours (Post) & 4.27 & 1.0e-4 (2.5e-3) & \textbf{0} & \textbf{0} \\
Ours (Train) & 5.64 & 7.0e-5 (2.0e-3) & \textbf{0} & \textbf{0} \\
Ours (Train-p) & 5.25 & \textbf{5.4e-5 (8.0e-4)} & \textbf{0} & \textbf{0} \\
\bottomrule
\end{tabular}
\caption{DCOPF Performance on IEEE 118-bus System. $^{\dagger}$ indicates constraint violations.}
\label{tab:dcopf}
\end{table}

Table~\ref{tab:dcopf} shows that T-SKM-Net achieves remarkable computational efficiency, delivering over 25$\times$ speedup compared to the traditional pandapower solver and about 10$\times$ speedup compared to DC3, while maintaining superior solution quality. The order-of-magnitude difference in computational time ($\sim$5ms vs 52.15ms) demonstrates the significant efficiency advantage of our approach. The neural network baseline, while fastest, exhibits obvious constraint violations, making it unsuitable for practical deployment. DC3 achieves equality constraint satisfaction but violates inequality constraints. Most importantly, our method under all 3 setups ensures strict constraint satisfaction with zero violations while maintaining competitive computational efficiency.

The results demonstrate T-SKM-Net's practical viability for real-time power system applications, where both computational speed and constraint satisfaction are critical requirements. The sub-10ms inference time enables rapid decision-making in dynamic grid operations, while zero constraint violations ensure system safety and operational feasibility.

\subsection{Predictive Portfolio Allocation}

To further evaluate the generality of T-SKM-Net beyond synthetic and power system benchmarks, we consider the predictive portfolio allocation task which maximizes the Sharpe ratio~\cite{sharpeSharpeRatioFall1998}. Denoting $x_i$ as the predicted portfolio decision variable of asset $i$, the portfolio is required to satisfy linear constraints
\[
\sum_{i=1}^{n} x_i = 1,\quad \sum_{i\in S} x_i \ge q,\quad 0 \le x_i \le 1,
\]
where $S$ is a set of preferred large-cap technology stocks (e.g., AAPL, MSFT, AMZN, TSLA, GOOGL) and $q=0.5$, as in~\cite{zengGLinSATGeneralLinear2024}. Performance is measured by the average Sharpe ratio over the test horizon.

We compare six methods that share the same network backbone and differ only in the constraint satisfaction layer and training strategy: NN (unconstrained network without any projection layer), qpth~\cite{amosOptNetDifferentiableOptimization2017} (differentiable QP layer), DC3~\cite{dontiDC3LearningMethod2021}, GLinSAT~\cite{zengGLinSATGeneralLinear2024}, Our method with both post-processing and joint-training mode. For all methods we use the same training and evaluation protocol as in~\cite{zengGLinSATGeneralLinear2024}. We report the average Sharpe ratio, the mean absolute inequality/equality violations (``Avg Ineq/Eq''), and the average projection time per test epoch.

\begin{table}[t]
\centering
\setlength{\tabcolsep}{3.5mm}
\small
\begin{tabular}{lcccc}
\toprule
\textbf{Method} & \textbf{Sharpe} & \textbf{\makecell{Avg\\Ineq/Eq ($10^{-2}$)}} & \textbf{\makecell{Proj.\\time (s)}} \\
\midrule
NN (no proj)   & 2.16 & 49.6 / 0.0  & 3.0e-4 \\
qpth           & 2.42 & 0.0 / 0.2  & 1.41 \\
DC3            & 2.97 & 8.4 / 0.0 & 1.1e-1 \\
GLinSAT        & 2.28 & \textbf{0.0 / 0.0}  & 2.6e-1 \\
Ours (Post)    & 2.93 & \textbf{0.0 / 0.0}  & 2.5e-1 \\
Ours (Train)   & \textbf{3.02} & \textbf{0.0 / 0.0} & 2.5e-1 \\
\bottomrule
\end{tabular}
\caption{Predictive portfolio allocation on the S\&P~500 dataset. ``Avg Ineq/Eq'' reports mean absolute inequality/equality violations. All methods share the same backbone; only the constraint satisfaction layer and training mode differ.}
\label{tab:snp500_portfolio}
\end{table}

Table~\ref{tab:snp500_portfolio} shows that the unconstrained NN baseline attains a moderate Sharpe ratio but exhibits large inequality violations, rendering many portfolios infeasible. qpth and DC3 significantly reduce one type of violation but still leave a non-negligible number of infeasible portfolios. GLinSAT achieves zero infeasibility under a $1\text{e-}3$ tolerance, but at the cost of a noticeably lower Sharpe ratio. In contrast, T-SKM-Net attains both strict feasibility and superior portfolio performance: even in post-processing mode it matches DC3 in Sharpe ratio while eliminating all infeasible portfolios, and in joint training mode it further improves the Sharpe ratio to $3.02$ with essentially zero average violations, at a projection cost comparable to GLinSAT and much lower than qpth. This experiment demonstrates that T-SKM-Net provides state-of-the-art performance on this task.

\subsection{Ablation Study}

Due to space limitations, we defer detailed ablation studies on null space transformation, sampling strategies, and SKM variants to the technical appendix, where we provide additional efficiency and accuracy comparisons.

\section{Conclusion}

This work presents T-SKM-Net, the first neural network framework to leverage Sampling Kaczmarz-Motzkin methods for linear constraint satisfaction. The proposed framework addresses the fundamental challenge of ensuring strict constraint satisfaction in neural networks while maintaining computational efficiency and end-to-end trainability.

The key contributions include: (1) integrating SKM methods to neural network constraint satisfaction for handling input-dependent dynamic linear constraints; (2) theoretical guarantees demonstrating L2 projection approximation in expectation and end-to-end trainability despite non-differentiable argmax operations; and (3) a flexible framework supporting both post-processing and joint training modes.

Experimental validation demonstrates significant computational advantages: 4-5$\times$ speedups compared to commercial Gurobi solver on L2 projection tasks, over 25$\times$ acceleration compared to pandapower solver on DC Optimal Power Flow with zero constraint violations, and superior performance compared to existing neural constraint satisfaction methods.

Several limitations warrant future investigation. The framework is currently restricted to linear constraints, and the selection of sampling size $\beta$ lacks theoretical guidance, relying on empirical rules. Additionally, backpropagation memory consumption scales linearly with iteration count. Future research directions include developing theoretical frameworks for optimal $\beta$ selection, exploring nonlinear constraint extensions (requiring fundamentally different approaches), and investigating memory-efficient optimization techniques such as implicit differentiation or selective computational graph preservation.

The T-SKM-Net framework demonstrates that classical iterative methods can be effectively integrated into modern deep learning pipelines, opening new possibilities for constraint-aware neural network design in safety-critical applications.

\section*{Acknowledgments}

This work was supported in part by the National Key R\&D Program of China (No. 2024YFB2409300).

\bibliography{aaai2026}

@article{raissiPhysicsinformedNeuralNetworks2019,
  title = {Physics-Informed Neural Networks: {{A}} Deep Learning Framework for Solving Forward and Inverse Problems Involving Nonlinear Partial Differential Equations},
  shorttitle = {Physics-Informed Neural Networks},
  author = {Raissi, M. and Perdikaris, P. and Karniadakis, G.E.},
  year = {2019},
  month = feb,
  journal = {Journal of Computational Physics},
  volume = {378},
  pages = {686--707},
  publisher = {Elsevier BV},
  issn = {0021-9991},
  doi = {10.1016/j.jcp.2018.10.045},
  urldate = {2025-07-22},
  copyright = {https://www.elsevier.com/tdm/userlicense/1.0/},
  langid = {english}
}

@inproceedings{amosOptNetDifferentiableOptimization2017,
  title = {{{OptNet}}: {{Differentiable Optimization}} as a {{Layer}} in {{Neural Networks}}},
  shorttitle = {{{OptNet}}},
  booktitle = {Proceedings of the 34th {{International Conference}} on {{Machine Learning}}, {{ICML}} 2017, {{Sydney}}, {{NSW}}, {{Australia}}, 6-11 {{August}} 2017},
  author = {Amos, Brandon and Kolter, J. Zico},
  editor = {Precup, Doina and Teh, Yee Whye},
  year = {2017},
  series = {Proceedings of {{Machine Learning Research}}},
  volume = {70},
  pages = {136--145},
  publisher = {PMLR},
  urldate = {2025-07-21}
}

@inproceedings{dontiDC3LearningMethod2021,
  title = {{{DC3}}: {{A}} Learning Method for Optimization with Hard Constraints},
  shorttitle = {{{DC3}}},
  booktitle = {9th {{International Conference}} on {{Learning Representations}}, {{ICLR}} 2021, {{Virtual Event}}, {{Austria}}, {{May}} 3-7, 2021},
  author = {Donti, Priya L. and Rolnick, David and Kolter, J. Zico},
  year = {2021},
  publisher = {OpenReview.net},
  urldate = {2025-07-21}
}

@misc{minHardNetHardConstrainedNeural2025,
  title = {{{HardNet}}: {{Hard-Constrained Neural Networks}} with {{Universal Approximation Guarantees}}},
  shorttitle = {{{HardNet}}},
  author = {Min, Youngjae and Azizan, Navid},
  year = {2025},
  month = jun,
  number = {arXiv:2410.10807},
  eprint = {2410.10807},
  primaryclass = {cs},
  publisher = {arXiv},
  doi = {10.48550/arXiv.2410.10807},
  urldate = {2025-07-21},
  archiveprefix = {arXiv}
}

@misc{tordesillasRAYENImpositionHard2023,
  title = {{{RAYEN}}: {{Imposition}} of {{Hard Convex Constraints}} on {{Neural Networks}}},
  shorttitle = {{{RAYEN}}},
  author = {Tordesillas, Jesus and How, Jonathan P. and Hutter, Marco},
  year = {2023},
  month = jul,
  number = {arXiv:2307.08336},
  eprint = {2307.08336},
  primaryclass = {cs},
  publisher = {arXiv},
  doi = {10.48550/arXiv.2307.08336},
  urldate = {2025-07-21},
  archiveprefix = {arXiv}
}

@misc{zhangEfficientLearningBasedSolver2024,
  title = {An {{Efficient Learning-Based Solver}} for {{Two-Stage DC Optimal Power Flow}} with {{Feasibility Guarantees}}},
  author = {Zhang, Ling and Tabas, Daniel and Zhang, Baosen},
  year = {2024},
  month = sep,
  number = {arXiv:2304.01409},
  eprint = {2304.01409},
  primaryclass = {eess},
  publisher = {arXiv},
  doi = {10.48550/arXiv.2304.01409},
  urldate = {2025-03-24},
  archiveprefix = {arXiv}
}

@misc{constante-floresEnforcingHardLinear2025,
  title = {Enforcing {{Hard Linear Constraints}} in {{Deep Learning Models}} with {{Decision Rules}}},
  author = {{Constante-Flores}, Gonzalo E. and Chen, Hao and Li, Can},
  year = {2025},
  month = may,
  number = {arXiv:2505.13858},
  eprint = {2505.13858},
  primaryclass = {cs},
  publisher = {arXiv},
  doi = {10.48550/arXiv.2505.13858},
  urldate = {2025-07-11},
  archiveprefix = {arXiv}
}

@article{loeraSamplingKaczmarzMotzkinAlgorithm2017,
  title = {A {{Sampling Kaczmarz-Motzkin Algorithm}} for {{Linear Feasibility}}},
  author = {Loera, Jes{\'u}s A. De and Haddock, Jamie and Needell, Deanna},
  year = {2017},
  journal = {SIAM J. Sci. Comput.},
  volume = {39},
  number = {5},
  doi = {10.1137/16M1073807}
}

@article{morshedSamplingKaczmarzMotzkin2022,
  title = {Sampling {{Kaczmarz Motzkin Method}} for {{Linear Feasibility Problems}}: {{Generalization}} \& {{Acceleration}}},
  shorttitle = {Sampling {{Kaczmarz Motzkin Method}} for {{Linear Feasibility Problems}}},
  author = {Morshed, Md Sarowar and Islam, Md Saiful and {Noor-E-Alam}, Md},
  year = {2022},
  month = jul,
  journal = {Mathematical Programming},
  volume = {194},
  number = {1-2},
  eprint = {2002.07321},
  primaryclass = {math},
  pages = {719--779},
  issn = {0025-5610, 1436-4646},
  doi = {10.1007/s10107-021-01649-8},
  urldate = {2025-04-28},
  archiveprefix = {arXiv}
}

@article{hornikMultilayerFeedforwardNetworks1989,
  title = {Multilayer Feedforward Networks Are Universal Approximators},
  author = {Hornik, Kurt and Stinchcombe, Maxwell and White, Halbert},
  year = {1989},
  month = jan,
  journal = {Neural Networks},
  volume = {2},
  number = {5},
  pages = {359--366},
  publisher = {Elsevier BV},
  issn = {0893-6080},
  doi = {10.1016/0893-6080(89)90020-8},
  urldate = {2025-07-29},
  copyright = {https://www.elsevier.com/tdm/userlicense/1.0/},
  langid = {english}
}

@article{cybenkoApproximationSuperpositionsSigmoidal1989,
  title = {Approximation by Superpositions of a Sigmoidal Function},
  author = {Cybenko, G.},
  year = {1989},
  month = dec,
  journal = {Mathematics of Control, Signals, and Systems},
  volume = {2},
  number = {4},
  pages = {303--314},
  issn = {0932-4194, 1435-568X},
  doi = {10.1007/BF02551274},
  urldate = {2025-07-29},
  copyright = {http://www.springer.com/tdm},
  langid = {english}
}

@article{smithNeuralNetworksCombinatorial1999,
  title = {Neural {{Networks}} for {{Combinatorial Optimization}}: {{A Review}} of {{More Than}} a {{Decade}} of {{Research}}},
  shorttitle = {Neural {{Networks}} for {{Combinatorial Optimization}}},
  author = {Smith, Kate A.},
  year = {1999},
  month = feb,
  journal = {INFORMS Journal on Computing},
  volume = {11},
  number = {1},
  pages = {15--34},
  issn = {1091-9856, 1526-5528},
  doi = {10.1287/ijoc.11.1.15},
  urldate = {2025-07-29},
  langid = {english}
}

@inproceedings{chenEnforcingPolicyFeasibility2021,
  title = {Enforcing {{Policy Feasibility Constraints}} through {{Differentiable Projection}} for {{Energy Optimization}}},
  booktitle = {Proceedings of the {{Twelfth ACM International Conference}} on {{Future Energy Systems}}},
  author = {Chen, Bingqing and Donti, Priya L. and Baker, Kyri and Kolter, J. Zico and Berg{\'e}s, Mario},
  year = {2021},
  month = jun,
  pages = {199--210},
  publisher = {ACM},
  address = {Virtual Event Italy},
  doi = {10.1145/3447555.3464874},
  urldate = {2025-07-29},
  isbn = {978-1-4503-8333-2},
  langid = {english}
}

@article{thurnerPandapowerOpenSourcePython2018,
  title = {Pandapower---{{An Open-Source Python Tool}} for {{Convenient Modeling}}, {{Analysis}}, and {{Optimization}} of {{Electric Power Systems}}},
  author = {Thurner, Leon and Scheidler, Alexander and Schafer, Florian and Menke, Jan-Hendrik and Dollichon, Julian and Meier, Friederike and Meinecke, Steffen and Braun, Martin},
  year = {2018},
  month = nov,
  journal = {IEEE Transactions on Power Systems},
  volume = {33},
  number = {6},
  pages = {6510--6521},
  issn = {0885-8950, 1558-0679},
  doi = {10.1109/TPWRS.2018.2829021},
  urldate = {2025-07-29},
  copyright = {https://ieeexplore.ieee.org/Xplorehelp/downloads/license-information/IEEE.html}
}

@article{JMLR:v24:21-0449,
  title = {Combinatorial Optimization and Reasoning with Graph Neural Networks},
  author = {Cappart, Quentin and Ch{\'e}telat, Didier and Khalil, Elias B. and Lodi, Andrea and Morris, Christopher and Veli{\v c}kovi{\'c}, Petar},
  year = {2023},
  journal = {Journal of Machine Learning Research},
  volume = {24},
  number = {130},
  pages = {1--61}
}

@article{lecunDeepLearning2015,
  title = {Deep Learning},
  author = {LeCun, Yann and Bengio, Yoshua and Hinton, Geoffrey},
  year = {2015},
  month = may,
  journal = {Nature},
  volume = {521},
  number = {7553},
  pages = {436--444},
  issn = {0028-0836, 1476-4687},
  doi = {10.1038/nature14539},
  urldate = {2025-07-29},
  langid = {english}
}

@article{doi:10.1137/18M1225409,
  title = {Physics-Informed Generative Adversarial Networks for Stochastic Differential Equations},
  author = {Yang, Liu and Zhang, Dongkun and Karniadakis, George Em},
  year = {2020},
  journal = {SIAM Journal on Scientific Computing},
  volume = {42},
  number = {1},
  eprint = {https://doi.org/10.1137/18M1225409},
  pages = {A292-A317},
  doi = {10.1137/18M1225409}
}

@article{butlerEfficientDifferentiableQuadratic2023,
  title = {Efficient Differentiable Quadratic Programming Layers: An {{ADMM}} Approach},
  shorttitle = {Efficient Differentiable Quadratic Programming Layers},
  author = {Butler, Andrew and Kwon, Roy H.},
  year = {2023},
  month = mar,
  journal = {Computational Optimization and Applications},
  volume = {84},
  number = {2},
  pages = {449--476},
  issn = {0926-6003, 1573-2894},
  doi = {10.1007/s10589-022-00422-7},
  urldate = {2025-07-29},
  langid = {english}
}

@book{golub2013matrix,
  title={Matrix computations},
  author={Golub, Gene H and Van Loan, Charles F},
  year={2013},
  publisher={JHU press}
}

@techreport{anderson2022real,
  title={A real-time operations manual for the IEEE 118 bus transmission model},
  author={Anderson, Alexander A and Kincic, Slaven and Jefferson, Brett A and Mcgary, Blaine J and Fallon, Corey K and Ciesielski, Danielle K and Wenskovitch, John E and Chen, Yousu},
  year={2022},
  institution={Pacific Northwest National Laboratory (PNNL), Richland, WA (United States)}
}

@article{carpentier1962contribution,
  title={Contribution a l'etude du dispatching economique},
  author={Carpentier, Jacques},
  journal={Bull. Soc. Fr. Elec. Ser.},
  volume={3},
  pages={431},
  year={1962}
}

@article{liu2024teaching,
  title={Teaching networks to solve optimization problems},
  author={Liu, Xinran and Lu, Yuzhe and Abbasi, Ali and Li, Meiyi and Mohammadi, Javad and Kolouri, Soheil},
  journal={IEEE Access},
  volume={12},
  pages={17102--17113},
  year={2024},
  publisher={IEEE}
}

@incollection{scutari2018parallel,
  title={Parallel and distributed successive convex approximation methods for big-data optimization},
  author={Scutari, Gesualdo and Sun, Ying},
  booktitle={Multi-Agent Optimization: Cetraro, Italy 2014},
  pages={141--308},
  year={2018},
  publisher={Springer}
}

@article{goda2023constructing,
  title={Constructing unbiased gradient estimators with finite variance for conditional stochastic optimization},
  author={Goda, Takashi and Kitade, Wataru},
  journal={Mathematics and Computers in Simulation},
  volume={204},
  pages={743--763},
  year={2023},
  publisher={Elsevier}
}

@article{robbins1951stochastic,
  title={A stochastic approximation method},
  author={Robbins, Herbert and Monro, Sutton},
  journal={The annals of mathematical statistics},
  pages={400--407},
  year={1951},
  publisher={JSTOR}
}

@inproceedings{bottou2010large,
  title={Large-scale machine learning with stochastic gradient descent},
  author={Bottou, L{\'e}on},
  booktitle={Proceedings of COMPSTAT'2010: 19th International Conference on Computational StatisticsParis France, August 22-27, 2010 Keynote, Invited and Contributed Papers},
  pages={177--186},
  year={2010},
  organization={Springer}
}

@book{goodfellow2016deep,
  title={Deep learning},
  author={Goodfellow, Ian and Bengio, Yoshua and Courville, Aaron and Bengio, Yoshua},
  year={2016},
  publisher={MIT press Cambridge}
}

@book{nocedal2006numerical,
  title={Numerical optimization},
  author={Nocedal, Jorge and Wright, Stephen J},
  year={2006},
  publisher={Springer}
}

@article{stellato2020osqp,
  title={OSQP: An operator splitting solver for quadratic programs},
  author={Stellato, Bartolomeo and Banjac, Goran and Goulart, Paul and Bemporad, Alberto and Boyd, Stephen},
  journal={Mathematical Programming Computation},
  volume={12},
  number={4},
  pages={637--672},
  year={2020},
  publisher={Springer}
}

@incollection{combettes2008fejer,
  title={Fej{\'e}r monotonicity in convex optimization},
  author={Combettes, Patrick L},
  booktitle={Encyclopedia of optimization},
  pages={1016--1024},
  year={2008},
  publisher={Springer}
}

@inproceedings{houGeneralizeLearnedHeuristics2023,
  title = {Generalize Learned Heuristics to Solve Large-Scale Vehicle Routing Problems in Real-Time},
  booktitle = {The {{Eleventh International Conference}} on {{Learning Representations}}},
  author = {Hou, Qingchun and Yang, Jingwei and Su, Yiqiang and Wang, Xiaoqing and Deng, Yuming},
  year = {2023}
}

@inproceedings{houLinearizedModelActive2018,
  title = {Linearized {{Model}} for {{Active}} and {{Reactive LMP Considering Bus Voltage Constraints}}},
  booktitle = {2018 {{IEEE Power}} \& {{Energy Society General Meeting}} ({{PESGM}})},
  author = {Hou, Qingchun and Zhang, Ning and Yang, Jingwei and Kang, Chongqing and Xia, Qing and Miao, Miao},
  year = {2018},
  month = aug,
  pages = {1--5},
  publisher = {IEEE},
  address = {Portland, OR},
  doi = {10.1109/PESGM.2018.8586191},
  urldate = {2025-08-01},
  isbn = {978-1-5386-7703-2}
}

@article{karczmarz1937angenaherte,
  title={Angenaherte auflosung von systemen linearer glei-chungen},
  author={Karczmarz, Stefan},
  journal={Bull. Int. Acad. Pol. Sic. Let., Cl. Sci. Math. Nat.},
  pages={355--357},
  year={1937}
}

@article{motzkin1954relaxation,
  title={The relaxation method for linear inequalities},
  author={Motzkin, Theodore Samuel and Schoenberg, Isaac Jacob},
  journal={Canadian Journal of Mathematics},
  volume={6},
  pages={393--404},
  year={1954},
  publisher={Cambridge University Press}
}

@article{strohmerRandomizedKaczmarzAlgorithm2009,
  title = {A {{Randomized Kaczmarz Algorithm}} with {{Exponential Convergence}}},
  author = {Strohmer, Thomas and Vershynin, Roman},
  year = {2009},
  month = apr,
  journal = {Journal of Fourier Analysis and Applications},
  volume = {15},
  number = {2},
  pages = {262--278},
  issn = {1069-5869, 1531-5851},
  doi = {10.1007/s00041-008-9030-4},
  urldate = {2025-08-01},
  copyright = {http://www.springer.com/tdm},
  langid = {english}
}

@inproceedings{zengGLinSATGeneralLinear2024,
  title = {{{GLinSAT}}: {{The General Linear Satisfiability Neural Network Layer By Accelerated Gradient Descent}}},
  shorttitle = {{{GLinSAT}}},
  booktitle = {Advances in {{Neural Information Processing Systems}} 38: {{Annual Conference}} on {{Neural Information Processing Systems}} 2024, {{NeurIPS}} 2024, {{Vancouver}}, {{BC}}, {{Canada}}, {{December}} 10 - 15, 2024},
  author = {Zeng, Hongtai and Yang, Chao and Zhou, Yanzhen and Yang, Cheng and Guo, Qinglai},
  editor = {Globersons, Amir and Mackey, Lester and Belgrave, Danielle and Fan, Angela and Paquet, Ulrich and Tomczak, Jakub M. and Zhang, Cheng},
  year = 2024,
  publisher = {arXiv},
  urldate = {2025-07-21}
}

@incollection{sharpeSharpeRatioFall1998,
  title = {The {{Sharpe Ratio}} ({{Fall}} 1994)},
  booktitle = {Streetwise},
  author = {Sharpe, William F.},
  editor = {Bernstein, Peter L. and Fabozzi, Frank J.},
  year = 1998,
  month = dec,
  pages = {169--178},
  publisher = {Princeton University Press},
  doi = {10.1515/9781400829408-022},
  urldate = {2025-11-15},
  isbn = {978-1-4008-2940-8}
}

@inproceedings{zhao2023ensuring,
  title = {Ensuring {{DNN}} Solution Feasibility for Optimization Problems with Linear Constraints},
  booktitle = {The Eleventh International Conference on Learning Representations, {{ICLR}} 2023, Kigali, Rwanda, May 1-5, 2023},
  author = {Zhao, Tianyu and Pan, Xiang and Chen, Minghua and Low, Steven H.},
  year = 2023,
  publisher = {OpenReview.net},
  bibsource = {dblp computer science bibliography, https://dblp.org}
}

@article{hanFRMNetFeasibilityRestoration2024a,
  title = {{{FRMNet}}: {{A Feasibility Restoration Mapping Deep Neural Network}} for {{AC Optimal Power Flow}}},
  shorttitle = {{{FRMNet}}},
  author = {Han, Jiayu and Wang, Wei and Yang, Chao and Niu, Mengyang and Yang, Cheng and Yan, Lei and Li, Zuyi},
  year = 2024,
  month = sep,
  journal = {IEEE Transactions on Power Systems},
  volume = {39},
  number = {5},
  pages = {6566--6577},
  issn = {0885-8950, 1558-0679},
  doi = {10.1109/TPWRS.2024.3354733},
  urldate = {2025-10-30},
  copyright = {https://ieeexplore.ieee.org/Xplorehelp/downloads/license-information/IEEE.html}
}

@inproceedings{liang2023low,
  title = {Low Complexity Homeomorphic Projection to Ensure Neural-Network Solution Feasibility for Optimization over (Non-)Convex Set},
  booktitle = {International Conference on Machine Learning, {{ICML}} 2023, 23-29 July 2023, Honolulu, Hawaii, {{USA}}},
  author = {Liang, Enming and Chen, Minghua and Low, Steven H.},
  editor = {Krause, Andreas and Brunskill, Emma and Cho, Kyunghyun and Engelhardt, Barbara and Sabato, Sivan and Scarlett, Jonathan},
  year = 2023,
  series = {Proceedings of Machine Learning Research},
  volume = {202},
  pages = {20623--20649},
  publisher = {PMLR},
  bibsource = {dblp computer science bibliography, https://dblp.org}
}

@article{nguyen2025fsnet,
  title = {{{FSNet}}: {{Feasibility-seeking}} Neural Network for Constrained Optimization with Guarantees},
  author = {Nguyen, Hoang T and Donti, Priya L},
  year = 2025,
  journal = {arXiv preprint arXiv:2506.00362},
  eprint = {2506.00362},
  archiveprefix = {arXiv}
}

\newpage
\def\isMainFile{}
% ------------------------------------------------
% LaTeX Template for ML/AI Research Paper
% ------------------------------------------------
\makeatletter
\@ifundefined{isMainFile}{
  % We are compiling a standalone document
  \newif\ifreproStandalone
  \reproStandalonetrue
}{
  % We are being \input into the main paper
  \newif\ifreproStandalone
  \reproStandalonefalse
}

\ifreproStandalone
\documentclass[11pt]{article}

% ------------------------------------------------
% Packages
% ------------------------------------------------
\usepackage{amsmath, amssymb}
\usepackage{graphicx}
\usepackage{natbib}          % For bibliography
\usepackage{geometry}        % For adjusting margins
\usepackage{times}           % For Times New Roman font
\usepackage{float}           % For controlling float positions
\usepackage{booktabs}        % For professional-looking tables
\usepackage{caption}
\usepackage{subcaption}
\usepackage{color}

\usepackage{amsthm}
\usepackage{amsmath}
\usepackage{multirow}

% Theorem environments
\newtheorem{theorem}{Theorem}
\newtheorem{lemma}{Lemma}
\newtheorem{proposition}{Proposition}
\newtheorem{corollary}{Corollary}
\newtheorem{definition}{Definition}
\newtheorem{remark}{Remark}
\newtheorem{assumption}{Assumption}

% ------------------------------------------------
% Page Setup
% ------------------------------------------------
\geometry{margin=1in}

% ------------------------------------------------
% Title and Author Information
% ------------------------------------------------
\title{Technical Appendix}
\author {}
% \author{
%   Your Name\thanks{Corresponding author: \texttt{your.email@example.com}} \\
%   \normalsize Your Affiliation \\
%   \and
%   Co-author Name \\
%   \normalsize Co-author Affiliation \\
%   % Add more authors if necessary
% }
\date{}

% ------------------------------------------------
% Document Begins
% ------------------------------------------------
\begin{document}
\pagestyle{empty}
\thispagestyle{empty}
\fi

\onecolumn

\ifreproStandalone
\maketitle
\else
  \begin{center}
    {\LARGE\bfseries Technical Appendix\par}
  \end{center}
\fi

\setcounter{section}{0}
\setcounter{subsection}{0}
\setcounter{figure}{0}
\setcounter{table}{0}
\setcounter{equation}{0}
\setcounter{theorem}{0}
\setcounter{lemma}{0}
\setcounter{corollary}{0}
\setcounter{assumption}{0}

\appendix

\section{Theoretical Proofs}
This section analyzes two core properties of the T-SKM-Net framework: effectiveness as a post-processing strategy and feasibility of end-to-end training. We prove three main theorems concerning the L2 projection approximation of the SKM algorithm, the L2 projection approximation of the entire framework, and the unbiasedness of gradient estimation. The first two theorems establish T-SKM-Net's effectiveness for constraint satisfaction post-processing, while the third validates its suitability for gradient-based neural network training.

\subsection{Effectiveness as Post-processing Strategy}

To validate T-SKM-Net as a post-processing strategy, we establish bounded approximation guarantees to the true L2 projection. The analysis first examines the core SKM algorithm for inequality constraints, then extends to the complete framework handling mixed constraint systems.

\subsubsection{L2 Projection Approximation of SKM Methods}

\begin{theorem}[L2 Projection Approximation of SKM Algorithm]
\label{thm:appendix-skm-l2}
Consider the inequality constraint $Az\leq b$ where $A\in \mathbb{R}^{p\times n}$ and $b \in \mathbb{R}^p$, with feasible region $\mathcal{P}=\{z: Az\leq b\}$. Let the SKM method start from initial point $z_{0}$ and obtain $z_{k}$ after $k$ iterations. Then the expected distance between the iterate and initial point satisfies:
$$\mathbb{E}[d(z_{k},z_{0})] \leq 2 \cdot d(z_{0},\mathcal{P})$$
where $d(z_0, \mathcal{P}) = \min_{z \in \mathcal{P}} \|z - z_0\|_2$ is the distance from the initial point to the feasible region.
\end{theorem}

\begin{proof}
We construct an auxiliary function $V(z) = \|z - P(z_0)\|^2$, where $P(z_0)$ is the L2 projection of $z_0$ onto the feasible region $\mathcal{P}$.

For the SKM single-step update 
$$z_{k+1} = z_k - \delta \frac{(a_{i^*}^T z_k - b_{i^*})_+}{\|a_{i^*}\|^2} a_{i^*}$$
we expand $V(z_{k+1})$:
\begin{align}
V(z_{k+1}) &= V(z_k) + \delta^2 \frac{(a_{i^*}^T z_k - b_{i^*})_+^2}{\|a_{i^*}\|^2} \nonumber\\
&\quad - 2\delta \frac{(a_{i^*}^T z_k - b_{i^*})_+}{\|a_{i^*}\|^2} \langle a_{i^*}, z_k - P(z_0) \rangle
\end{align}

Since $P(z_0) \in \mathcal{P}$, we have $a_{i^*}^T P(z_0) \leq b_{i^*}$, therefore:
$$\langle a_{i^*}, z_k - P(z_0) \rangle = a_{i^*}^T z_k - a_{i^*}^T P(z_0) \geq a_{i^*}^T z_k - b_{i^*}$$

When $a_{i^*}^T z_k > b_{i^*}$, we have $(a_{i^*}^T z_k - b_{i^*})_+ = a_{i^*}^T z_k - b_{i^*}$, thus:
$$\frac{(a_{i^*}^T z_k - b_{i^*})_+}{\|a_{i^*}\|^2} \langle a_{i^*}, z_k - P(z_0) \rangle \geq \frac{(a_{i^*}^T z_k - b_{i^*})_+^2}{\|a_{i^*}\|^2}$$

Substituting this yields:
$$V(z_{k+1}) \leq V(z_k) - \delta(2-\delta) \frac{(a_{i^*}^T z_k - b_{i^*})_+^2}{\|a_{i^*}\|^2}$$

When $0 < \delta < 2$, we have $V(z_{k+1}) \leq V(z_k)$, indicating that the sequence $\{V(z_k)\}$ is monotonically decreasing.

Taking expectation and applying Jensen's inequality:
\begin{align}
\mathbb{E}[V(z_k)] &\leq V(z_0) = d(z_0, \mathcal{P})^2\\
\mathbb{E}[\|z_k - P(z_0)\|] &\leq \sqrt{\mathbb{E}[\|z_k - P(z_0)\|^2]} \leq d(z_0, \mathcal{P})
\end{align}

Applying the triangle inequality:
\begin{align}
\mathbb{E}[d(z_k, z_0)] &\leq \mathbb{E}[\|z_k - P(z_0)\|] + \|P(z_0) - z_0\| \\
&\leq d(z_0, \mathcal{P}) + d(z_0, \mathcal{P}) = 2d(z_0, \mathcal{P})
\end{align}
\end{proof}

\subsubsection{L2 Projection Approximation of T-SKM-Net Framework}

We now extend the analysis to the complete T-SKM-Net framework, examining how nullspace transformation for equality constraints interacts with subsequent SKM processing of inequality constraints.

\begin{theorem}[L2 Projection Approximation of T-SKM-Net Framework]
\label{thm:appendix-tskm-l2}
Consider the mixed linear constraint system $A(x)z \leq b(x)$ and $C(x)z = d(x)$, with feasible region $\mathcal{F}(x) = \{z : A(x)z \leq b(x), C(x)z = d(x)\}$. Let $P(y_0)$ be the L2 projection of initial point $y_0$ onto feasible region $\mathcal{F}(x)$, and $N$ be the nullspace basis matrix of equality constraint matrix $C(x)$. When the T-SKM-Net framework executes $k$ iterations, the expected distance between output $z_k$ and initial point $y_0$ satisfies:
$$\mathbb{E}[d(z_k, y_0)] \leq \sqrt{1+4\kappa(N)^2} \cdot d(y_0, P(y_0))$$
where $\kappa(N) = \frac{\sigma_{\max}(N)}{\sigma_{\min}(N)}$ is the condition number of nullspace basis matrix $N$. In particular, when $N$ is the nullspace basis matrix obtained via SVD, the condition number $\kappa(N)=1$, thus:
$$\mathbb{E}[d(z_k, y_0)] \leq \sqrt{5} \cdot d(y_0, P(y_0))$$
\end{theorem}

\begin{proof}
According to the T-SKM-Net algorithm workflow, the general solution of equality constraint $C(x)z = d(x)$ is expressed through SVD decomposition as:
$$z = z_{\text{proj}} + Nw$$
where $z_{\text{proj}} = y_0 - C^{\dagger}(Cy_0 - d)$ and $N$ is the nullspace basis matrix.

Define the projection error vector $e = y_0 - z_{\text{proj}}$. Due to projection properties, $e \in \text{row}(C)$ and $Nw \in \text{null}(C)$, thus $e \perp Nw$.

The L2 projection of the original problem can be expressed as $P(y_0) = z_{\text{proj}} + Nw^*$, where $w^* = \arg\min_{w \in \mathcal{P}_w} \|Nw\|^2$ and $\mathcal{P}_w = \{w : (AN)w \leq b - Az_{\text{proj}}\}$ is the transformed feasible region.

Orthogonal decomposition of distance:
$$d(y_0, P(y_0))^2 = \|y_0 - (z_{\text{proj}} + Nw^*)\|^2 = \|e\|^2 + \|Nw^*\|^2$$

In the transformed constraint space, the SKM algorithm starts from $w_0 = 0$. Let $w_{\text{proj}} = \arg\min_{w \in \mathcal{P}_w} \|w\|^2$ be the L2 projection of the subproblem.

According to Theorem~\ref{thm:appendix-skm-l2}, the SKM algorithm in the subspace satisfies:
$$\mathbb{E}[\|w_k\|] \leq 2\|w_{\text{proj}}\|$$

Since $w^* \in \mathcal{P}_w$ and $w_{\text{proj}}$ minimizes $\|w\|^2$, we have $\|w_{\text{proj}}\| \leq \|w^*\|$.

Combined with the matrix singular value property $\|w^*\| \leq \frac{1}{\sigma_{\min}(N)}\|Nw^*\|$, we obtain:
$$\|w_{\text{proj}}\| \leq \frac{1}{\sigma_{\min}(N)}\|Nw^*\|$$

Using orthogonality and the triangle inequality:
$$\mathbb{E}[d(z_k, y_0)] = \mathbb{E}[\|e + Nw_k\|] \leq \|e\| + \mathbb{E}[\|Nw_k\|]$$

Further estimation:
\begin{align}
\mathbb{E}[\|Nw_k\|] &\leq \|N\|_2 \mathbb{E}[\|w_k\|] \nonumber\\
&\leq 2\|N\|_2 \|w_{\text{proj}}\| \nonumber\\
&= 2 \sigma_{\max}(N) \|w_{\text{proj}}\| \nonumber\\
&\leq 2 \sigma_{\max}(N) \frac{1}{\sigma_{\min}(N)} \|Nw^*\| \nonumber\\
&= 2\kappa(N)\|Nw^*\|
\end{align}

Therefore:
$$\mathbb{E}[d(z_k, y_0)] \leq \|e\| + 2\kappa(N)\|Nw^*\|$$

By the Cauchy-Schwarz inequality:
\begin{align}
&\|e\| + 2\kappa(N)\|Nw^*\| \nonumber\\
&\leq \sqrt{1 + 4\kappa(N)^2} \cdot \sqrt{\|e\|^2 + \|Nw^*\|^2} \nonumber\\
&= \sqrt{1 + 4\kappa(N)^2} \cdot d(y_0, P(y_0))
\end{align}

Thus:
$$\mathbb{E}[d(z_k, y_0)] \leq \sqrt{1 + 4\kappa(N)^2} \cdot d(y_0, P(y_0))$$

In particular, when $N$ is the nullspace basis matrix obtained via SVD, since $N$ is part of the right singular matrix with $N^TN=I$, we have $\kappa(N)=1$, therefore:
$$\mathbb{E}[d(z_k, y_0)] \leq \sqrt{5} \cdot d(y_0, P(y_0))$$
\end{proof}

\subsection{End-to-End Trainability Analysis}

For T-SKM-Net integration as differentiable layers, we validate reliable gradient estimation despite non-differentiable operations like $\arg\max$ constraint selection. This analysis establishes T-SKM-Net's suitability for end-to-end neural network training.

Since the constraint transformation process (SVD decomposition and linear transformation) in the T-SKM-Net framework is differentiable, the core challenge for end-to-end trainability lies in handling non-differentiable operations in the SKM algorithm. This section analyzes the gradient estimation properties of the SKM algorithm, proving that despite the existence of non-differentiable $\arg\max$ operations, the algorithm still supports unbiased gradient estimation.

\subsubsection{Problem Setup}

Consider a parameterized linear inequality system $A(x)w \leq b(x)$, where $A(x) \in \mathbb{R}^{p \times m}$, $b(x) \in \mathbb{R}^p$, and $x \in \mathbb{R}^{n_{\text{in}}}$ is the input parameter. The SKM algorithm executes $K$ iterations: $w_0 = 0$, $w_{k+1} = F(w_k, A(x), b(x), \mathcal{S}_k)$, where:

\begin{enumerate}
\item $\mathcal{S}_k \subset \{1,\ldots,p\}$ is the set of $\beta$ constraint indices sampled at step $k$ (with replacement)
\item $F(w, A, b, \mathcal{S}) = w - \delta \frac{r_{i^*}^+}{\|a_{i^*}\|^2} a_{i^*}$
\item $i^* = \arg\max_{i \in \mathcal{S}} (a_i^T w - b_i)$, $r_{i^*} = a_{i^*}^T w - b_{i^*}$, $(z)^+ = \max(0,z)$
\end{enumerate}

Let $\omega = (\mathcal{S}_0, \ldots, \mathcal{S}_{K-1})$ be the sampling sequence, and $W_K(x,\omega)$ be the algorithm output.

\begin{assumption}[Non-degeneracy]
\label{assump:nondeg}
There exists a constant $c > 0$ such that $\|a_i(x)\| \geq c$ for all $i,x$.
\end{assumption}

\begin{assumption}[Differentiability]
\label{assump:diff}
$A(x), b(x)$ are differentiable with respect to $x$, and there exist constants $L_A, L_b > 0$ such that $\|\nabla_x A(x)\| \leq L_A$ and $\|\nabla_x b(x)\| \leq L_b$.
\end{assumption}

\begin{assumption}[Boundedness]
\label{assump:bound}
There exists $M > 0$ such that $\|A(x)\|, \|b(x)\| \leq M$ for all relevant $x$.
\end{assumption}

\begin{assumption}[Step Size]
\label{assump:step}
The step size parameter satisfies $0 < \delta < 2$.
\end{assumption}

\begin{assumption}[Non-degeneracy of Tie Events]
\label{assump:tie}
For any fixed $\omega, k$ and any relevant constraint indices $i,j$, the functions $(a_i(x) - a_j(x))^T W_k(x,\omega) - (b_i(x) - b_j(x))$ and $a_i(x)^T W_k(x,\omega) - b_i(x)$ do not vanish identically on any open subset of input space $\mathbb{R}^{n_{\text{in}}}$.
\end{assumption}

\begin{theorem}[End-to-End Trainability of SKM Algorithm]
\label{thm:appendix-trainability}
Under Assumptions~\ref{assump:nondeg}--\ref{assump:tie}, the SKM algorithm has the following trainability properties:

\textbf{(i) Well-defined Expected Gradient}: $\mathbb{E}_\omega[W_K(x,\omega)]$ is almost everywhere differentiable with respect to $x$, and the expected gradient $\nabla_x \mathbb{E}_\omega[W_K(x,\omega)]$ is well-defined.

\textbf{(ii) Finite Gradient Variance}: $\text{Var}(\nabla_x W_K(x,\omega)) < \infty$.

\textbf{(iii) Unbiased Gradient Estimation}: The path derivative computed by automatic differentiation is an unbiased estimator of the expected gradient:
$$\mathbb{E}_\omega[\nabla_x W_K(x,\omega)] = \nabla_x \mathbb{E}_\omega[W_K(x,\omega)]$$
\end{theorem}

\begin{proof}
We establish the result through a series of lemmas.

\begin{lemma}[Boundedness of Iteration Sequence]
\label{lem:boundedness}
Based on the proof of Theorem~\ref{thm:appendix-skm-l2}, the SKM algorithm has Fejér monotonicity property. The iteration sequence $\{w_k\}$ satisfies: for any feasible point $w^* \in \mathcal{C} = \{w : A(x)w \leq b(x)\}$,
$$\|w_{k+1} - w^*\|^2 \leq \|w_k - w^*\|^2 - \delta(2-\delta)\frac{(r_{i^*}^+)^2}{\|a_{i^*}\|^2}$$
Therefore, there exists a constant $B$ such that $\sup_{k \geq 0} \|w_k\| \leq B < \infty$.
\end{lemma}

\begin{lemma}[Jacobian Properties of Single-step Mapping]
\label{lem:jacobian}
The Jacobian of single-step mapping $F(w,A,b,\mathcal{S})$ satisfies:

\textbf{(1) With respect to w}: When $a_{i^*}^T w > b_{i^*}$ and $\arg\max$ is unique:
$$\frac{\partial F}{\partial w} = I - \delta P_{a_{i^*}}, \quad \text{where} \quad P_a = \frac{aa^T}{\|a\|^2}$$
Since $P_a$ is a rank-1 projection matrix with eigenvalues $\{1, 0, \ldots, 0\}$:
$$\left\|\frac{\partial F}{\partial w}\right\|_2 = \max\{|1-\delta|, 1\} = 1 \quad (0 < \delta < 2)$$

\textbf{(2) With respect to b}: When constraints are active:
$$\left\|\frac{\partial F}{\partial b}\right\| \leq \frac{\delta}{c} := C_b$$

\textbf{(3) With respect to A}: Using Lemma~\ref{lem:boundedness} with $\|w\| \leq B$ and $|r_{i^*}^+| \leq MB + M$:
$$\left\|\frac{\partial F}{\partial A}\right\| \leq \delta \left( \frac{B}{c} + \frac{3(MB + M)}{c^2} \right) := C_A$$
\end{lemma}

\begin{lemma}[Zero Measure of Tie Events]
\label{lem:tie}
Under Assumption~\ref{assump:tie}, define the tie event
$$T_k = \{(x,\omega): \exists i \neq j \in \mathcal{S}_k, a_i^T W_k(x,\omega) - b_i = a_j^T W_k(x,\omega) - b_j\}$$
Then $\mu(T_k) = 0$, where $\mu$ is the Lebesgue measure.
\end{lemma}

\begin{proof}[Proof of Lemma~\ref{lem:tie}]
For fixed $\omega$, $W_k(x,\omega)$ is a piecewise differentiable function of $x$. The condition $a_i^T W_k - b_i = a_j^T W_k - b_j$ is equivalent to $(a_i - a_j)^T W_k = b_i - b_j$, which under Assumption~\ref{assump:tie} defines hypersurfaces in $x$ space. Similarly, the boundary where a single constraint transitions from non-violation to violation $a_i^T W_k = b_i$ also forms hypersurfaces.

Since $|\mathcal{S}_k| = \beta$ is finite, there are only finitely many constraint pairs $(i,j)$, thus $T_k$ is a union of finitely many hypersurfaces. Each hypersurface has zero Lebesgue measure in $n_{\text{in}}$-dimensional space, and a finite union of zero-measure sets remains zero-measure, hence $\mu(T_k) = 0$.

Similarly, $\mu(\bigcup_{k=0}^{K-1} T_k) = 0$ as it is a finite union of zero-measure sets.
\end{proof}

Now we prove the main theorem parts:

\textbf{(i) Well-defined Expected Gradient}: 
Through mathematical induction, at non-tie points we apply the chain rule:
$$\nabla_x W_{k+1} = \frac{\partial F}{\partial w}\nabla_x W_k + \frac{\partial F}{\partial A}\nabla_x A + \frac{\partial F}{\partial b}\nabla_x b$$

Let $C = C_A L_A + C_b L_b$. From Lemma~\ref{lem:jacobian}, we obtain the recurrence:
$$\|\nabla_x W_{k+1}\| \leq 1 \cdot \|\nabla_x W_k\| + C$$
Therefore: $\|\nabla_x W_k\| \leq kC$.

Setting $D_K = KC < \infty$, by the dominated convergence theorem and Lemma~\ref{lem:tie}:
$$\nabla_x \mathbb{E}_\omega[W_K(x,\omega)] = \mathbb{E}_\omega[\nabla_x W_K(x,\omega)]$$

\textbf{(ii) Finite Gradient Variance}: 
$$\mathbb{E}[\|\nabla_x W_K\|^2] \leq (KC)^2 < \infty$$

\textbf{(iii) Unbiased Gradient Estimation}: Since the probability $P(\omega)$ of sampling sequence $\omega$ is independent of parameter $x$:
$$\nabla_x \mathbb{E}_\omega[W_K(x,\omega)] = \sum_\omega P(\omega) \nabla_x W_K(x,\omega) = \mathbb{E}_\omega[\nabla_x W_K(x,\omega)]$$
\end{proof}

\begin{corollary}
The SKM constraint satisfaction layer in the T-SKM-Net framework supports standard automatic differentiation, providing unbiased, finite-variance gradient estimates of the loss function with respect to upstream network parameters, thus ensuring the feasibility of end-to-end training.
\end{corollary}

\section{Hyperparameter Selection Analysis}

To provide practical guidance for T-SKM-Net deployment, we conduct systematic experiments analyzing the impact of key hyperparameters: step size $\delta$ and sampling size $\beta$. The experiments are performed on 5000-dimensional problems with mixed linear constraints to evaluate both computational efficiency and solution quality.

\subsection{Step Size $\delta$ Analysis}

Figure~\ref{fig:delta_analysis} presents the relationship between step size $\delta$ and algorithm performance across different T-SKM-Net variants. The results reveal a clear trade-off between computational efficiency and approximation quality.

\begin{figure}[htbp]
    \centering
    \includegraphics[width=0.8\textwidth]{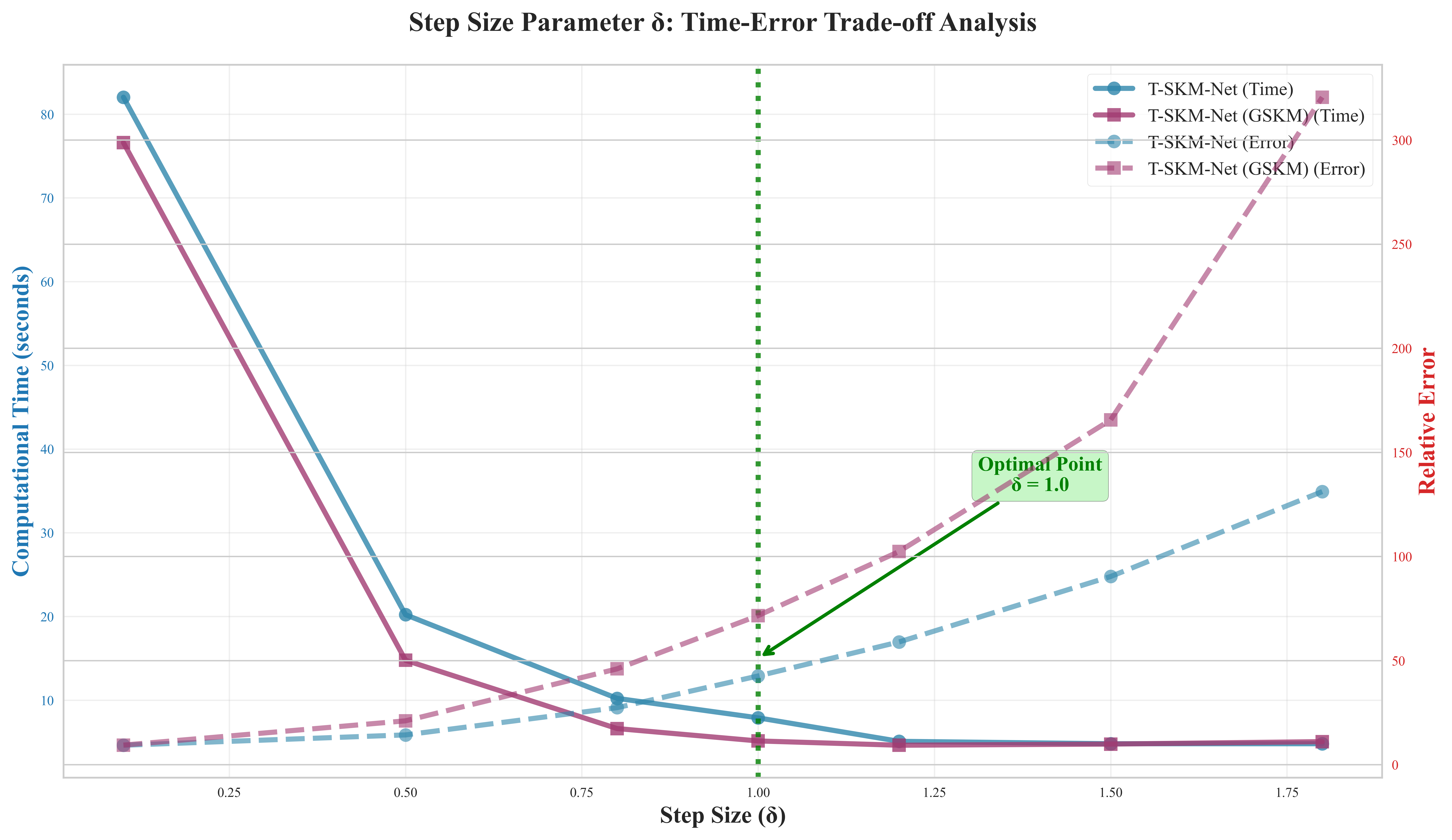}
    \caption{Step size $\delta$ analysis showing the trade-off between computational time (top) and approximation error (bottom). Results demonstrate that $\delta = 1.0$ represents an optimal balance point.}
    \label{fig:delta_analysis}
\end{figure}

The experimental results demonstrate that $\delta = 1.0$ serves as a critical \textbf{inflection point} in the performance characteristics:

\begin{itemize}
    \item \textbf{Computational Efficiency}: For $\delta < 1.0$, computational time decreases rapidly as $\delta$ increases. Beyond $\delta = 1.0$, further increases yield diminishing returns in speed improvement, with the rate of acceleration significantly reduced.
    
    \item \textbf{Approximation Quality}: The approximation error exhibits exponential growth for $\delta > 1.0$. While $\delta = 1.0$ maintains reasonable error levels ($\sim$42.6), increasing to $\delta = 1.2$ nearly doubles the error to $\sim$59.0, and $\delta = 1.8$ results in error exceeding 131.
    
    \item \textbf{Stability Across Variants}: This inflection point behavior is consistent across all T-SKM-Net algorithmic variants (GSKM, MSKM, NSKM), indicating the robustness of this empirical finding.
\end{itemize}

Based on these observations, we \textbf{recommend $\delta = 1.0$ as the default step size}, providing an optimal balance between computational efficiency and L2 projection approximation quality. This choice aligns with theoretical analysis suggesting that step sizes approaching the upper bound of the convergence condition ($\delta < 2$) may compromise solution quality.

\subsection{Sampling Size $\beta$ Analysis}

Figure~\ref{fig:beta_analysis} illustrates the effect of sampling size $\beta$ on computational performance. Unlike the step size analysis, the sampling size exhibits more complex behavior without a single optimal value.

\begin{figure}[htbp]
    \centering
    \includegraphics[width=0.8\textwidth]{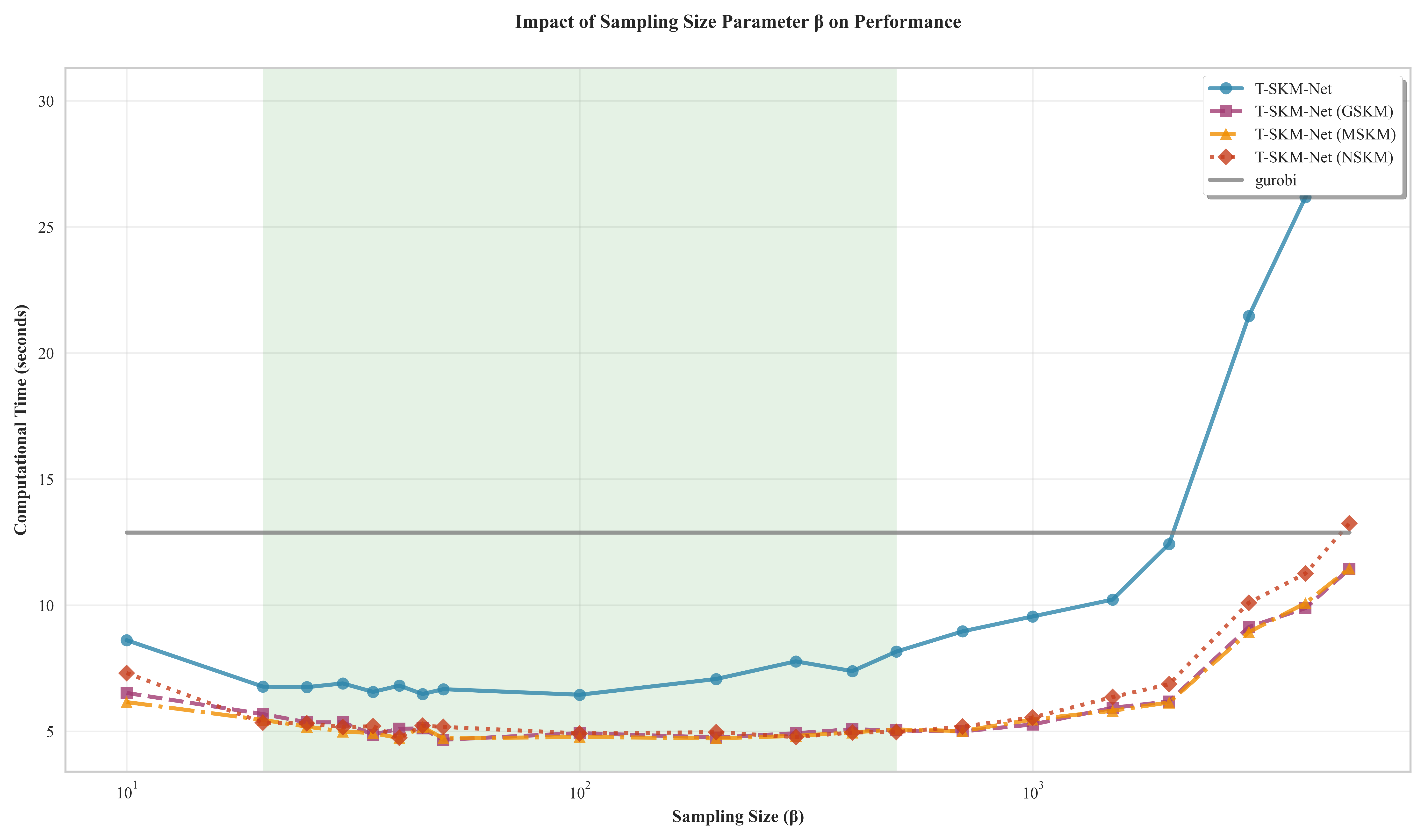}
    \caption{Sampling size $\beta$ analysis showing computational time across different T-SKM-Net variants. The results indicate that both extremely small and extremely large sampling sizes lead to increased computational cost.}
    \label{fig:beta_analysis}
\end{figure}

The experimental results reveal several important patterns:

\begin{itemize}
    \item \textbf{Small Sampling Size Effect}: When $\beta$ is very small (close to 0), computational time increases significantly. This occurs because insufficient sampling leads to poor constraint selection, requiring more iterations to achieve convergence.
    
    \item \textbf{Large Sampling Size Effect}: When $\beta$ approaches the total number of inequality constraints $m$, computational time also increases due to the overhead of processing larger constraint sets in each iteration.
    
    \item \textbf{Moderate Range Stability}: For moderate values of $\beta$ (approximately 25-500), computational time remains relatively stable, suggesting that the algorithm is robust within this range.
    
    \item \textbf{Variant-Dependent Behavior}: Different algorithmic variants (GSKM, MSKM, NSKM) show similar trends but with different absolute performance levels, indicating that momentum acceleration and other enhancements can improve efficiency regardless of sampling strategy.
\end{itemize}

While we cannot establish a definitive theoretical optimal value for $\beta$, the empirical evidence suggests \textbf{avoiding extreme values}. We recommend:

\begin{itemize}
    \item \textbf{Small-scale problems}: $\beta = \max(5, m/10)$
    \item \textbf{Large-scale problems}: $\beta = \max(10, O(\sqrt{m}))$
\end{itemize}

where $m$ is the number of inequality constraints. This strategy ensures sufficient constraint coverage while avoiding excessive computational overhead.

\subsection{Practical Implementation Guidelines}

Based on the comprehensive hyperparameter analysis, we provide the following implementation guidelines:

\begin{enumerate}
    \item \textbf{Default Configuration}: Use $\delta = 1.0$ and $\beta = \max(10, O(\sqrt{m}))$ as starting points for most applications.

    \item \textbf{Application-Specific Tuning}: For applications requiring higher precision, consider slightly reducing $\delta$ to 0.8-0.9 at the cost of increased computational time.

\end{enumerate}

These empirical guidelines provide practitioners with concrete parameter settings that balance computational efficiency and solution quality across diverse application scenarios.

\section{Ablation Studies}

To provide comprehensive validation of the T-SKM-Net framework components, we conduct extensive ablation studies examining the impact of nullspace transformation, sampling strategies, and algorithmic variants. Note that these experiments utilize a different problem generator than the main paper experiments to ensure broader experimental coverage - we provide both code implementations for reproducibility.

Due to space limitations, the main paper only briefly mentioned ablation studies. Here we present a complete ablation analysis to provide thorough component validation and practical implementation guidance for researchers and practitioners. The experiments cover different dimensional ranges depending on the specific analysis: nullspace transformation comparison (10-3,000 dimensions), sampling strategy analysis (10-12,000 dimensions), and algorithmic variants comparison (10-12,000 dimensions).

\subsection{Impact of Nullspace Transformation}

Figure~\ref{fig:ablation_null_space} demonstrates the critical importance of nullspace transformation in handling mixed constraint systems. The comparison between T-SKM-Net (with nullspace transformation) and the naive approach (without nullspace transformation) reveals substantial performance differences.

\begin{figure}[htb]
    \centering
    \includegraphics[width=0.9\textwidth]{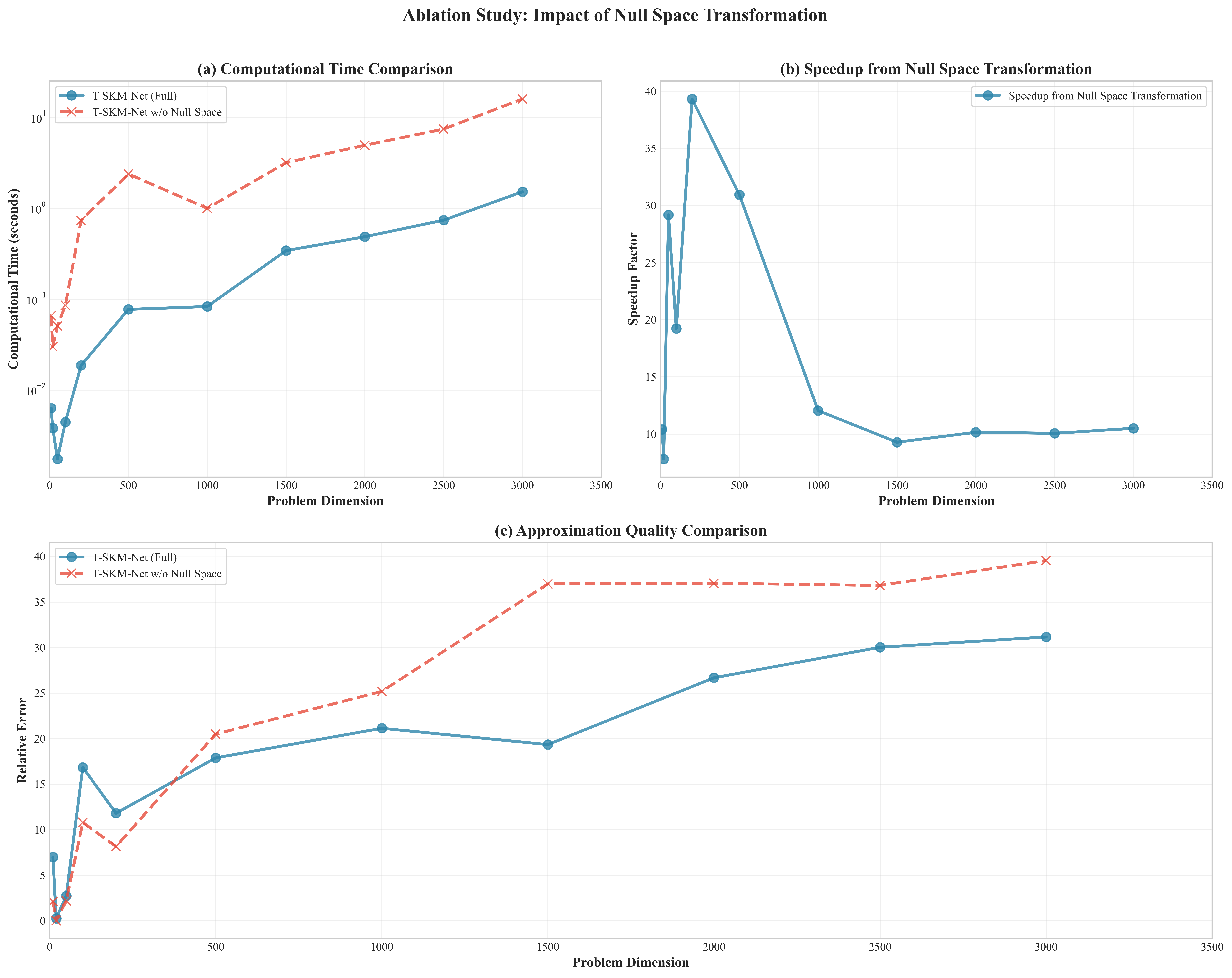}
    \caption{Ablation study showing the impact of nullspace transformation. Panel (a) shows computational time comparison, panel (b) demonstrates the speedup achieved through nullspace transformation, and panel (c) compares the approximation quality of both approaches.}
    \label{fig:ablation_null_space}
\end{figure}

The experimental results confirm our theoretical analysis from Section 4.1 regarding the geometric mismatch between equality and inequality constraints:

\begin{itemize}
    \item \textbf{Computational Efficiency}: The nullspace transformation provides substantial speedups across all tested dimensions. At dimension 3000, T-SKM-Net achieves approximately 10.5$\times$ acceleration compared to the naive approach (1.52s vs 15.97s).
    
    \item \textbf{Consistent Improvement}: The speedup factor remains stable across different problem scales, consistently delivering 10-20$\times$ acceleration. For example, at dimension 1000, the acceleration is approximately 12$\times$ (0.083s vs 0.999s), and at dimension 2000, it achieves 10.1$\times$ speedup (0.487s vs 4.94s).
    
    \item \textbf{Solution Quality Enhancement}: As shown in panel (c), nullspace transformation significantly improves L2 projection approximation quality. The relative error remains consistently lower for T-SKM-Net across all dimensions, with the naive approach showing substantially higher approximation errors, particularly for larger problems.
\end{itemize}

This ablation study validates our framework design choice and demonstrates that nullspace transformation provides consistent and substantial improvements in both computational efficiency and solution quality for practical constraint satisfaction problems.

\subsection{Sampling Strategy Analysis}

Figure~\ref{fig:ablation_sampling} compares two sampling strategies: sampling with replacement versus sampling without replacement. This comparison provides \textbf{practical implementation guidance} for developers working with the T-SKM-Net framework.

\begin{figure}[htbp]
    \centering
    \includegraphics[width=0.7\textwidth]{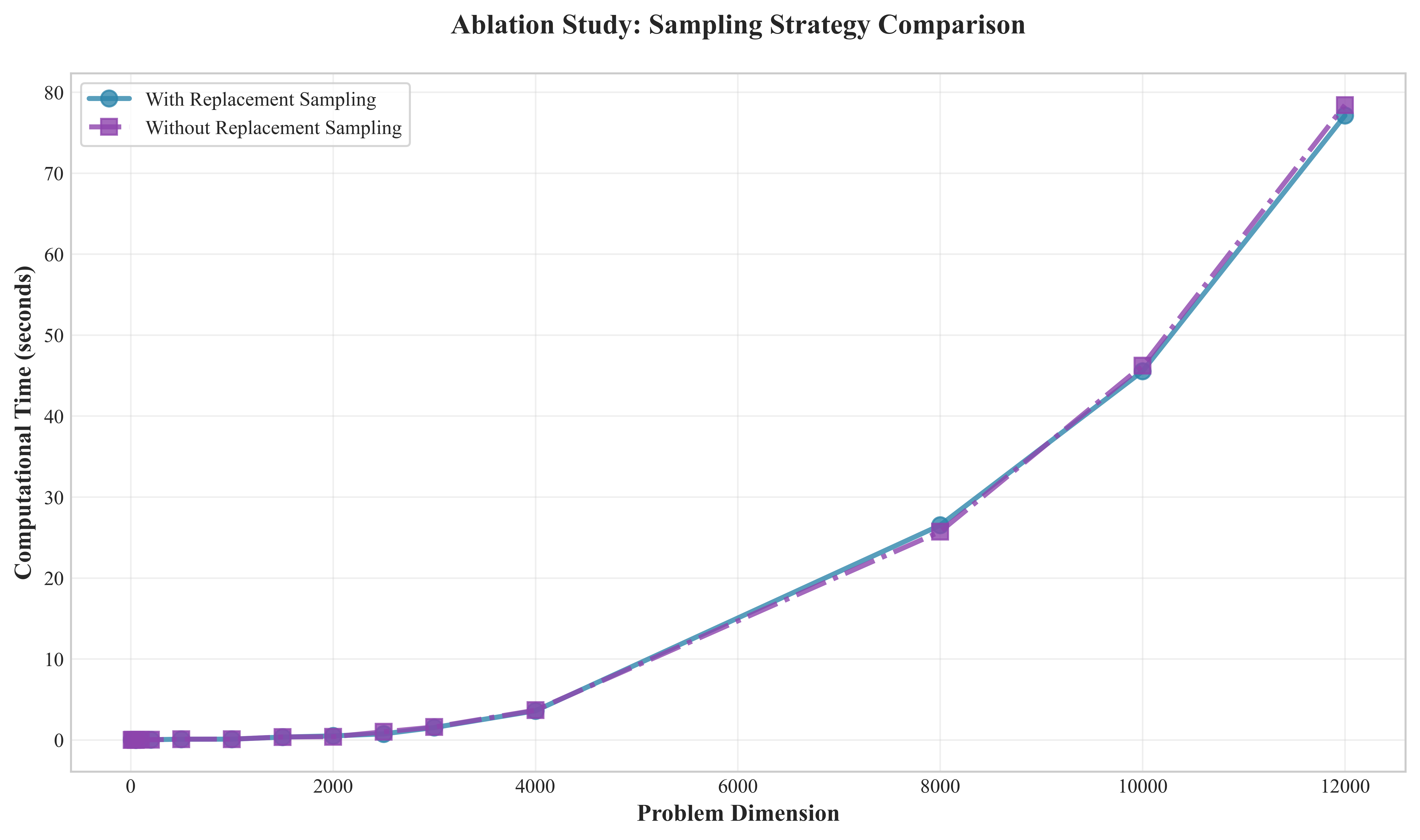}
    \caption{Comparison of sampling strategies showing computational time. Both strategies achieve comparable performance, validating the use of simpler with-replacement sampling in practical implementations.}
    \label{fig:ablation_sampling}
\end{figure}

The results provide important insights for practical implementation:

\begin{itemize}
    \item \textbf{Similar Performance}: Both strategies exhibit nearly identical computational time scaling across different problem dimensions, with minimal performance differences.
    
    \item \textbf{Implementation Simplicity}: For batch processing, sampling with replacement can be implemented with a single line: \texttt{torch.randint(0, p, (batch\_size, beta))}, directly creating batch-dimensional indices. In contrast, sampling without replacement requires iterating over each batch element with \texttt{torch.randperm()}, making the implementation more complex and less GPU-friendly.
    
    \item \textbf{Practical Recommendation}: Given comparable performance and significantly simpler batch implementation, developers can confidently use with-replacement sampling without sacrificing algorithmic effectiveness while maintaining clean, efficient code.
\end{itemize}

This analysis demonstrates that the implementation choice between sampling strategies has minimal impact on performance, allowing practitioners to prioritize code simplicity and maintainability.

\subsection{Algorithmic Variants Comparison}

As mentioned in Section 4.5 of the main paper, the T-SKM-Net framework supports various SKM algorithm variants, including momentum-enhanced versions. To validate this flexibility and guide algorithm selection, we compare several momentum-enhanced SKM variants against traditional optimization approaches within the nullspace transformation framework.

Figure~\ref{fig:ablation_variants} presents a comprehensive comparison of different algorithmic approaches, contrasting our SKM-based methods against traditional optimization solvers.

\begin{figure}[htbp]
    \centering
    \includegraphics[width=0.9\textwidth]{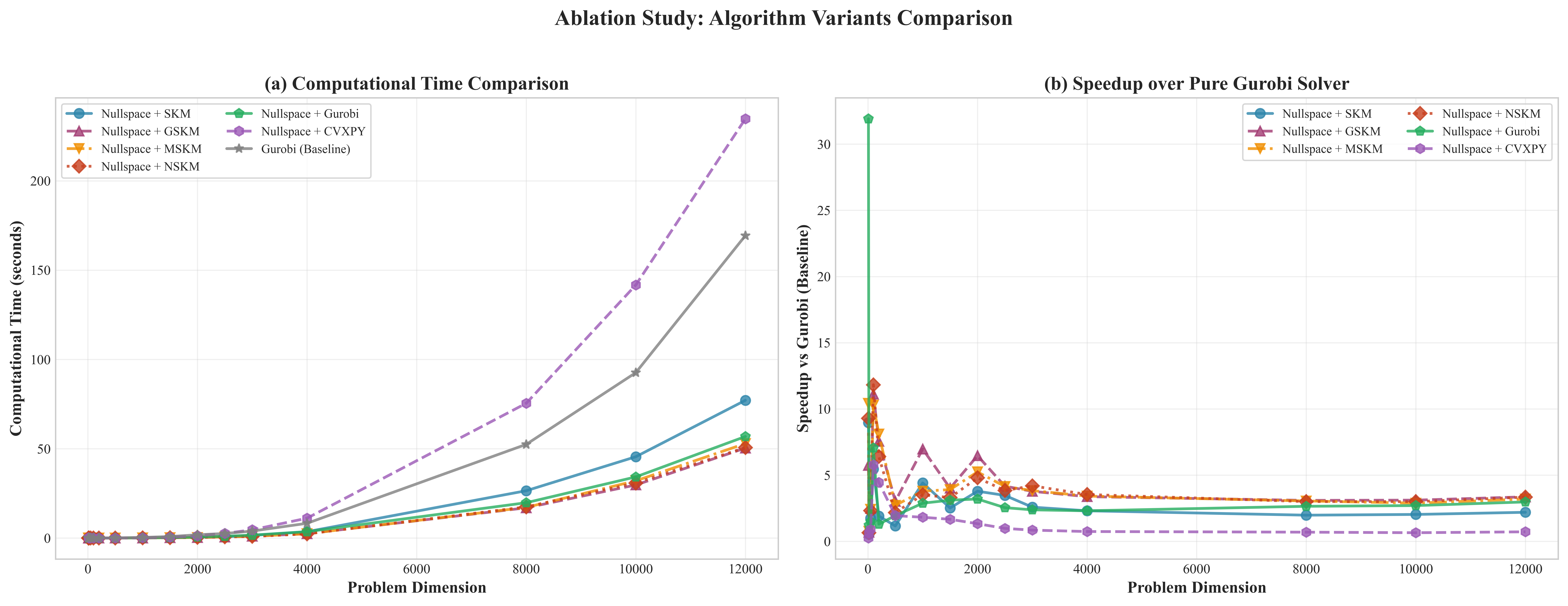}
    \caption{Comparison of algorithmic variants showing computational time (left) and speedup over pure Gurobi solver (right). Momentum-enhanced SKM variants achieve superior performance compared to traditional optimization methods within the nullspace framework.}
    \label{fig:ablation_variants}
\end{figure}

The tested variants include:
\begin{itemize}
    \item \textbf{Basic T-SKM-Net}: Standard SKM implementation as described in the main paper
    \item \textbf{GSKM}: Uses weighted averaging of current and previous iterations: $x = (1-\xi) z_{\text{curr}} + \xi z_{\text{prev}}$ with $\xi = -0.25$, providing stabilized convergence through iteration combination
    \item \textbf{NSKM}: Incorporates Nesterov momentum acceleration with momentum parameter $\mu = 0.25$, using look-ahead gradients for improved convergence rates
    \item \textbf{MSKM}: Employs classical heavy-ball momentum (also known as polyak momentum) with $\mu = 0.25$, maintaining velocity from previous iterations to accelerate convergence
    \item \textbf{Traditional Solver Integration}: Nullspace transformation combined with conventional optimization solvers (Gurobi and CVXPY)
\end{itemize}

The comparison reveals important insights for T-SKM-Net algorithm selection:

\begin{itemize}
    \item \textbf{Momentum Acceleration Benefits}: All momentum-enhanced T-SKM-Net variants (GSKM, NSKM, MSKM) consistently outperform the basic T-SKM-Net implementation, demonstrating the effectiveness of momentum acceleration in iterative constraint satisfaction.
    
    \item \textbf{Superior Performance Over Open-Source Solvers}: Our proposed T-SKM-Net methods significantly outperform nullspace transformation combined with open-source solvers. For instance, "Nullspace + CVXPY" performs substantially worse than even the baseline, highlighting the advantages of the iterative SKM approach over traditional optimization within the constraint satisfaction framework.
    
    \item \textbf{Competitive with Commercial Solvers}: Remarkably, momentum-enhanced T-SKM-Net variants outperform even commercial solver combinations. At 12,000 dimensions, the momentum-enhanced variants (50.4-52.9s) surpass "Nullspace + Gurobi" (56.9s), demonstrating that our iterative approach can compete with state-of-the-art commercial optimization software.
    
    \item \textbf{Scalability Advantages}: The performance advantage of T-SKM-Net becomes more pronounced with increasing problem dimension, demonstrating superior scalability for large-scale constraint satisfaction applications.
    
    \item \textbf{Momentum Enhancement Effectiveness}: Among all tested approaches, momentum-enhanced T-SKM-Net variants consistently achieve the best performance, validating the importance of momentum acceleration in the SKM framework.
\end{itemize}

These results validate that T-SKM-Net, particularly with momentum enhancement, provides a highly effective approach for neural network constraint satisfaction, outperforming both open-source optimization methods and competing favorably with commercial solvers while offering superior scalability for large-scale applications.

% ------------------------------------------------
% Document Ends
% ------------------------------------------------
\ifreproStandalone
\end{document}
\fi

\end{document}